%% file: main.tex
\documentclass{article} 

\usepackage[breaklinks=true]{hyperref}
\usepackage{url}

\usepackage[utf8]{inputenc} 
\usepackage[truedimen,margin=30mm]{geometry} 

\usepackage[T1]{fontenc}    
\usepackage{booktabs}       
\usepackage{amsfonts}       
\usepackage{nicefrac}       
\usepackage{microtype}      

\usepackage{amsmath}
\usepackage{amssymb}
\usepackage{amsthm}
\usepackage{multirow}
\usepackage{here}
\usepackage{wrapfig}
\usepackage{centernot}
\usepackage{color}
\usepackage{comment}
\usepackage{natbib}
\usepackage{breakcites}

\newtheorem{theorem}{Theorem}
\newtheorem{lemma}{Lemma}
\newtheorem{proposition}{Proposition}
\newtheorem{definition}{Definition}
\newtheorem{corollary}{Corollary}
\newtheorem{assumption}{Assumption}

\usepackage{color}
\usepackage{CJKutf8}
\usepackage{ifthen}
\newcommand{\COMM}[2]{{
\begin{CJK}{UTF8}{ipxm}
\ifthenelse{\equal{#1}{AN}}{\color{red}}{
\ifthenelse{\equal{#1}{TS}}{\color{blue}}{
\ifthenelse{\equal{#1}{BB}}{\color{green}}}}
[#1: #2]
\end{CJK}
}}

\input{notations}

\title{Gradient Descent can Learn Less Over-parameterized Two-layer Neural Networks on Classification Problems}

\author{Atsushi Nitanda$^{1,2,3,\dag}$, Geoffrey Chinot$^{4,\ddag}$, Taiji Suzuki$^{1,2,\star}$
\vspace{2mm}\\
\normalsize{\textit{$^1$Graduate School of Information Science and Technology, The University of Tokyo, Japan}} \\
\normalsize{\textit{$^2$Center for Advanced Intelligence Project, RIKEN, Japan}} \\
\normalsize{\textit{$^3$PRESTO, Japan Science and Technology Agency, Japan}} \\
\normalsize{\textit{$^4$ENSAE, CREST, France}} \\
\small{Email: $^\dag$nitanda@mist.i.u-tokyo.ac.jp, $^\ddag$geoffrey.chinot@ensae.fr, $^\star$taiji@mist.i.u-tokyo.ac.jp}} 

\date{}

\begin{document}

\maketitle

\begin{abstract}
Recently, several studies have proven the global convergence and generalization abilities of the gradient descent method for two-layer ReLU networks.
Most studies especially focused on the regression problems with the squared loss function, except for a few, and the importance of the positivity of the {\it neural tangent kernel} has been pointed out.
On the other hand, the performance of gradient descent on classification problems using the logistic loss function has not been well studied, and further investigation of this problem structure is possible.
In this work, we demonstrate that the separability assumption using a {\it neural tangent} model is more reasonable than the positivity condition of the neural tangent kernel and provide a refined convergence analysis of the gradient descent for two-layer networks with smooth activations.
A remarkable point of our result is that our convergence and generalization bounds have much better dependence on the network width in comparison to related studies.
Consequently, our theory provides a generalization guarantee for less over-parameterized two-layer networks, while most studies require much higher over-parameterization.
\end{abstract}

\input{main_body}

\section*{Acknowledgment}
AN was partially supported by JSPS Kakenhi (19K20337) and JST-PRESTO.
TS was partially supported by JSPS Kakenhi (15H05707, 18H03201, and 18K19793), Japan Digital Design and JST-CREST.

\bibliographystyle{iclr2020_conference}

\input{main.bbl}

\appendix
\section*{Appendix}
\input{supplement_body}

\end{document}

%% file: notations.tex
\newcommand{\defeq}{\overset{def}{=}}

\newcommand{\argmax}{\operatornamewithlimits{argmax}}

\def\pd<#1>{\left\langle #1 \right\rangle}
\def\floor[#1]{\left\lfloor #1 \right\rfloor}
\def\ceil[#1]{\left\lceil #1 \right\rceil}

\newcommand{\realsp}{\mathbb{R}}

\newcommand{\posintegers}{\mathbb{Z}_{+}}
\newcommand{\expec}{\mathbb{E}}
\newcommand{\prob}{\mathbb{P}}

\newcommand{\fdim}{d}

\newcommand{\ndata}{n}
\newcommand{\featuresp}{\mathcal{X}}
\newcommand{\labelsp}{\mathcal{Y}}
\newcommand{\tpr}{\nu}

\newcommand{\radcomp}{\Re}

\def\risk{\mathcal{L}}

\def\hilsp{\mathcal{H}}

%% file: main_body.tex
\section{Introduction}\label{sec:introduction}
In recent years, many studies have been devoted to explaining the great success of over-parameterized neural networks, where the number of parameters is much larger 
than that needed to fit a given training dataset.
On the other hand, this study treats less over-parameterized two-layer neural networks using smooth activation functions 
and analyzes the convergence and generalization abilities of the gradient descent method for optimizing this type of network. 

For over-parameterized two-layer neural networks, 
\cite{du2018gradient,arora2019fine,chizat2018global,mei2018mean} showed the global convergence of the gradient descent.
These studies are mainly divided into two groups depending on the scaling factor of the output of the networks 
to which the global convergence property has been demonstrated using different types of proofs.
For the scaling factor $1/m$, ($m$: the number of hidden units), \cite{chizat2018global,mei2018mean} showed 
the convergence to the global minimum over probability measures when $m \rightarrow \infty$ 
by utilizing the Wasserstein gradient flow perspective \citep{nitanda2017stochastic} on the gradient descent.
For the scaling factor $1/m^\beta$ ($\beta < 1$), 
\cite{du2018gradient} essentially demonstrated that the kernel smoothing of functional gradients
by the {\it neural tangent kernel} \citep{jacot2018neural,chizat2018note} has comparable performance with the functional gradient as $m\rightarrow \infty$ 
by making a positivity assumption on the Gram-matrix of this kernel, 
resulting in the global convergence property. 
In addition, \cite{arora2019fine} provided a generalization bound via a fine-grained analysis of the gradient descent.
These studies provide the first steps to understand the role of over-parameterization of neural networks and the gradient descent on regression problems using the squared loss function.
For the classification problems with logistic loss, a few studies \citep{allen2018learning,cao2019generalization,cao2019generalization_b} investigated the convergence and generalization abilities of gradient descent under a separability assumption with a suitable model instead of the positivity of the neural tangent kernel.
In this study, we further develop this line of research on binary classification problems. 

\paragraph{Our contributions.}
We provide fine-grained global convergence and generalization analyses of the gradient descent for two-layer neural networks with smooth activations 
under a separability assumption with a sufficient margin using a {\it neural tangent model}, which is a non-linear model with feature extraction through a neural tangent.
We demonstrate that a separability assumption is more suitable than the positivity condition of the neural tangent kernel because (i) the positive neural tangent kernel leads to weak separability and conversely, (ii) separability leads to the positivity of the neural tangent kernel only on a cone spanned by labels, which is very small space compared to the whole space.
Therefore, the separability condition is rather weak in this sense but it is enough to ensure global convergence for the classification problems.
Thus, a significantly improved convergence and generalization analyses with respect to network width can be obtained because the positivity of the neural tangent kernel is not required.
Consequently, our theory provides a generalization guarantee for less over-parameterized two-layer networks trained by gradient descent, 
while most existing results relying on the positive neural tangent kernel essentially require high over-parameterization.
To the best of our knowledge, this is the first work that shows the global convergence and the generalization guarantees for neural networks without over-parameterization on classification problems with logistic loss in the literature. \footnote{Following the initial version of our manuscript, \cite{ji2019polylogarithmic,chen2019much} have provided the global convergence and generalization guarantee for ReLU networks with polylogarithmic width based on our separability assumption.}
Most studies \citep{allen2018learning,cao2019generalization,cao2019generalization_b} have focused on highly over-parameterized neural networks with ReLU activation,
and less over-parameterized settings have been considered difficult for showing the global convergence property of gradient descent.
However, we note that these studies provided global convergence and generalization analyses of the (stochastic) gradient descent for (deep) ReLU networks by making a similar but different assumption than ours.
Thus, our and these studies do not include each other because of the difference of the network structure (i.e., network depth and activation type) and assumptions.

We here describe the main result informally.
A neural tangent model is an infinite-dimensional non-linear model using transformed features $(\partial_\theta \sigma (\theta^{(0)\top}x))_{\theta^{(0)}\sim \mu_0}$, where 
$\sigma$ is a smooth activation and $\mu_0$ is a distribution used to 
initialize the parameters of the input layer in two-layer neural networks.
Theorem \ref{theorem:informal_main_result} states that gradient descent can find an $\epsilon$-accurate solution in terms of the expected classification error for a wide class of over-parameterized two-layer neural networks under a separability assumption using a neural tangent model.
\begin{theorem}[Informal] \label{theorem:informal_main_result}
Suppose that a given data distribution is separable by a neural tangent model with a sufficient margin under $L_\infty$-constraint.
If for any $\epsilon > 0$, the hyperparameters satisfy  one of the following
\begin{align*}
&\textrm{(i)}\ \ \beta \in [0,1),\ m=\Omega(\epsilon^{\frac{-1}{1-\beta}}),\ T=\Theta(\epsilon^{-2}),\ \eta=\Theta(m^{2\beta-1}),\ \ndata = \tilde{\Omega}(\epsilon^{-4}), \\
&\textrm{(ii)}\ \ \beta = 0,\ m=\tilde{\Theta}(\epsilon^{-3/2}),\ T=\tilde{\Theta}(\epsilon^{-1}),\ \eta=\Theta(m^{-1}),\ \ndata = \tilde{\Omega}(\epsilon^{-2}).
\end{align*}
then with high probability over the random initialization and choice of samples of size $n$, the gradient descent with a learning rate $\eta$ achieves an expected $\epsilon$-classification error within $T$-iterations.
\end{theorem}

\begin{table*}[t]
\caption{Summary of hyperparameter settings and assumptions to achieve an expected $\epsilon$-classification error by gradient descent for binary classifications.
The ``Separability'' column denotes the types of models where a separability assumption is made.
$m$ is the number of hidden units, $n$ is the size of the training data, and $T$ is the number of iterations of gradient descent.
The notations $\tilde{\Omega}$ and $\tilde{\Theta}$ hide the logarithmic terms in the big-$\Omega$ and -$\Theta$ notations.
Smooth activations include {\it sigmoid, tanh, swish} activations, and several smooth approximations of ReLU.
As for \cite{allen2018learning,cao2019generalization,cao2019generalization_b,chen2019much}, we pick up results specialized to two-layer networks.}
\begin{center}
  {
  \small
  \begin{tabular}{ccccccc}
    \toprule
    & Activation
    & Separability
    & Deep
    & $m$
    & $n$
    & $T$ \\
    \midrule    
    \cite{allen2018learning}
    & ReLU
    & Smooth Target
    & yes
    & $\tilde{\Omega}(\epsilon^{-10})$
    & $\Omega(\epsilon^{-4})$ 
    & $\tilde{\Theta}(\epsilon^{-2})$ \\
    \hline
    \cite{cao2019generalization}
    & ReLU
    & ReLU NN
    & yes      
    & $\tilde{\Omega}(\epsilon^{-14})$
    & $\tilde{\Omega}(\epsilon^{-4})$ 
    & $\tilde{\Theta}(\epsilon^{-2})$ \\
    \hline
    \cite{cao2019generalization_b}
    & ReLU
    & ReLU NN
    & yes      
    & $\tilde{\Omega}(\epsilon^{-14})$
    & $\tilde{\Omega}(\epsilon^{-2})$ 
    & $\tilde{\Theta}(\epsilon^{-2})$ \\
    \hline
    \cite{ji2019polylogarithmic}
    & ReLU
    & Neural Tangent
    & no
    & $\Omega(\mathrm{polylog}(n,\epsilon))$
    & $\tilde{\Omega}(\epsilon^{-2})$ 
    & $\Theta(\epsilon^{-1})$ \\
    \hline   
    \cite{chen2019much}
    & ReLU
    & Neural Tangent
    & yes
    & $\Omega(\mathrm{polylog}(n,\epsilon))$
    & $\tilde{\Omega}(\epsilon^{-2})$ 
    & $\Theta(\epsilon^{-1})$ \\
    \hline   
    \rule{0mm}{4mm}
    \multirow{2}{*}{{\bf This work}}
    & \multirow{2}{*}{Smooth}
    & \multirow{2}{*}{Neural Tangent}
    & \multirow{2}{*}{no}      
    &$\Omega(\epsilon^{-1})$
    &$\tilde{\Omega}(\epsilon^{-4})$ 
    &$\Theta(\epsilon^{-2})$ \\
    &
    &
    &
    &$\tilde{\Theta}(\epsilon^{-3/2})$
    &$\tilde{\Omega}(\epsilon^{-2})$ 
    &$\tilde{\Theta}(\epsilon^{-1})$ \\    
    \bottomrule
  \end{tabular}
  }
\end{center}
\label{table:comparison_of_hyperparameters_settings}
\end{table*}

\paragraph{Related work.}
A few recent studies \citep{allen2018learning,cao2019generalization,cao2019generalization_b} are closely related to our work because they also treated the logistic loss function. 
As stated above, problem settings in our and these studies are somewhat different, but we compare our result with those specialized to two-layer network to argue the reasonableness of the separability assumption by a neural tangent model. 
More recently, following the initial version of our manuscript, \cite{ji2019polylogarithmic,chen2019much} have shown that polylogarithmic network width suffices for ReLU networks to achieve an arbitrary small classification errors based on our separability assumption.
On the other hand, separability assumptions were made on an infinite-width two-layer ReLU network in \cite{cao2019generalization,cao2019generalization_b} and on a smooth target function in \cite{allen2018learning}. 
For generalization analyses, our study and \cite{ji2019polylogarithmic,chen2019much} exhibit much better dependency on the network width owing to a better problem setting with a fine-grained analysis.
Table \ref{table:comparison_of_hyperparameters_settings} provides a comparison of the hyperparameter settings of networks and gradient descent in related studies to achieve an expected $\epsilon$-classification error.
As evident in Table \ref{table:comparison_of_hyperparameters_settings}, our theory and \cite{ji2019polylogarithmic,chen2019much} ensure the same generalization ability as those of \cite{allen2018learning,cao2019generalization,cao2019generalization_b} for a more comprehensive class of two-layer networks with respect to the network width.

\paragraph{Other related work.}
\cite{brutzkus2017sgd,li2018learning} provided generalization analyses of the stochastic gradient descent for two-layer networks.
\cite{brutzkus2017sgd} assumed that datasets are linear separable and this restrictive assumption was relaxed to mixtures of well separated data distributions in \cite{li2018learning}.
As for the convergence rate analysis, the network width $m=\tilde{\Omega}(\epsilon^{-24})$ 
and the number of samples (iterations) $\ndata=\Theta(T)=\tilde{O}(\epsilon^{-12})$ are required 
to achieve an expected $\epsilon$-classification error in \cite{li2018learning}.
In \cite{allen2018convergence,zou2018stochastic}, the global convergence analyses in terms of optimization without the specification of network size were provided, but we note that better generalization bounds cannot be obtained from these results because the complexities of neural networks also cannot be specified. 

Apart from the abovementioned studies, there are many other studies 
\citep{brutzkus2017globally,zhong2017recovery,tian2017analytical,soltanolkotabi2017learning,du2018gradient,zhang2018learning,arora2019fine,oymak2019towards,zhang2019fast,wu2019global} 
that focus on regression problems.
Based on the positivity condition of the neural tangent kernel, \cite{du2018gradient,arora2019fine,zhang2019fast,wu2019global} showed the global convergence of gradient descent methods for neural networks with the required network widths $\Omega(n^6)$ \citep{du2018gradient, wu2019global}, $\Omega(n^7 \epsilon^{-2})$ \citep{arora2019fine}, and $\Omega(n^4)$ \citep{zhang2019fast}.
Because sample complexities are generally slower than or equal to $n = \Omega(\epsilon^{-2})$, these network widths are very large compared to results for classification problems.  
As stated above, the reason of such improvement on the network width for classification problems is that the property of logistic loss can lead to a more reasonable assumption (i.e., separability assumption by a neural tangent model) than the positivity assumption of the neural tangent kernel.
Moreover, we would like to emphasize that proof techniques are quite different for the squared loss and the logistic loss functions because the latter function lacks the strong convexity. Thus, we cannot utilize the linear convergence property for the logistic loss and parameters will diverge, which also causes the difficulty of showing better generalization ability without a fine-grained analysis.

\section{Preliminary}\label{sec:preliminary}
Here, we describe the problem setting for the binary logistic regression and discuss the functional gradients to provide a clear theoretical view of the gradient methods for two-layer neural networks.

\subsection{Problem Setting}
Let $\featuresp = \realsp^\fdim$ and $\labelsp$ be a feature space and the set of binary labels $\{-1,1\}$, respectively.
We denote by $\tpr$ a true probability measure on $\featuresp \times \labelsp$ and by $\tpr_{\ndata}$ an empirical probability measure, deduced from observations $(x_i,y_i)_{i=1}^\ndata$ independently drawn from $\tpr$,
i.e., $d\tpr_{\ndata}(X,Y)=\sum_{i=1}^\ndata \delta_{(x_i,y_i)}(X,Y)dXdY/\ndata$, where $\delta$ is the Dirac delta function.
The marginal distributions of $\tpr$ and $\tpr_\ndata$ on $X$ are denoted by $\tpr^X$ and $\tpr_{\ndata}^X$, respectively.
For $\zeta \in \realsp$ and $y\in \labelsp$, let $l(\zeta,y)$ be the logistic loss: $\log(1+\exp(-y\zeta))$.
Then, the objective function to be minimized is formalized as follows:
\begin{equation*}
\risk(\Theta) \defeq \expec_{(X,Y) \sim \tpr_\ndata}[ l(f_\Theta(X),Y)] = \frac{1}{n}\sum_{i=1}^\ndata l( f_\Theta(x_i),y_i),
\end{equation*}
where $f_\Theta : \featuresp \rightarrow \realsp$ is a two-layer neural network equipped with parameters $\Theta=(\theta_r)_{r=1}^m$.
When we consider a function $f_\Theta$ as a variable of the objective function, we denote $\risk(f_\Theta) \defeq \risk(\Theta)$.

The two-layer neural network treated in this study is formalized as follows.
For parameters $\Theta = (\theta_r)_{r=1}^m$ ($\theta_r \in \realsp^\fdim$) and fixed constants $(a_r)_{r=1}^m \in \{-1,1\}^m$: 
\begin{equation}\label{eq:2nn}
f_\Theta(x) = \frac{1}{m^\beta} \sum_{r=1}^m a_r \sigma(\theta_r^\top x),
\end{equation}
where $m$ is the number of hidden units, $\beta$ is an order of the scaling factor, and $\sigma: \realsp \rightarrow \realsp$ is a smooth activation function such as sigmoid, tanh, swish \citep{ramachandran2017}), and other smooth approximations of ReLU.
In the training procedure, the parameters $\Theta=(\theta_r)_{r=1}^m$ of the input layer are optimized.
This setting is the same as those in \cite{du2018gradient,arora2019fine,zhang2019fast,wu2019global}, except for the types of activation functions, scaling factor, and loss function.

\subsection{Functional Gradient} 
We denote by $L_2(\tpr_{\ndata}^X)$ the function space from $\featuresp$ to $\realsp$, equipped with the inner product $\pd<\cdot,\cdot>_{L_2(\tpr_\ndata^X)}$:
\begin{equation*}
  \pd<\phi,\psi>_{L_2(\tpr_\ndata^X)}
    \defeq \expec_{X \sim \tpr_\ndata^X}\left[ \phi(X)\psi(X) \right], \hspace{3mm} \forall \phi, \forall \psi \in L_2(\tpr_\ndata^X).
\end{equation*}
Following the tradition in the literature of boosting and kernel methods, 
we call $L_2(\tpr_\ndata^X)$ the function space, although this space is actually an $n$-dimensional space because the cardinality 
of the support of $\tpr_\ndata^X$ is $n$.
The key notion to explain the behavior of the gradient descent is the functional gradient in this function space $L_2(\tpr_\ndata^X)$.
We define the functional gradient at a predictor $f: \featuresp \rightarrow \realsp$ as, 
\begin{align*}
    \nabla_f \risk(f)(x) \defeq  
 \begin{cases}
   \left. \partial_{\zeta}l(\zeta,y_i)\right|_{\zeta=f(x_i)} & (x=x_i), \\
    0 & (\mathrm{otherwise}). 
  \end{cases} 
\end{align*}
This is simply a Fr\'{e}chet differential (functional gradient) in $L_2(\tpr_{\ndata}^X)$.
That is, it follows that
\[ \risk(f+\phi) = \risk(f) + \pd< \nabla_f \risk(f),\phi>_{L_2(\tpr_{\ndata}^X)}
    + o(\|\phi\|_{L_2(\tpr_{\ndata}^X)}), \hspace{3mm}\forall f,\forall \phi \in L_2(\tpr_{\ndata}^X). \]
Therefore, the functional gradient descent using $\nabla_f \risk(f)$ directly optimizes $\risk$ in a function space $L_2(\tpr_{\ndata}^X)$ and converges to a global minimum because the objective function $\risk$ is convex with respect to a function $f$.
However, because $\nabla_f \risk(f)$ contains no information regarding the unseen data, this method is meaningless in terms of generalization.
Thus, some smoothing techniques are required to guarantee the generalization.
The gradient descent method for two-layer neural networks can be recognized as a type of kernel-smoothed functional gradient using the {\it neural tangent kernel} \citep{jacot2018neural}, and this perspective is significantly useful in showing the global convergence because it characterizes the behavior of the vanilla gradient descent in a function space.

\section{Brief Review of Functional Gradient Methods}\label{sec:overview}
Functional gradient methods have been mainly studied for gradient boosting \citep{mason1999boosting,friedman2001greedy} and 
kernel methods \citep{kivinen2004online,smale2006online,ying2006online,raskutti2014early,wei2017early} in the machine learning community, 
but more recently, it has been found to be useful in explaining the behavior of gradient descent for neural networks \citep{jacot2018neural,chizat2018note,du2018gradient,allen2018convergence,arora2019fine}.
Our analysis is also heavily based on the functional gradient perspective of gradient descent.
Thus, we briefly review the functional gradient methods.

In gradient boosting, $\nabla_f \risk(f)$ is approximated by finding a similar function in weak learners $\mathcal{G}$:
\begin{equation} \label{eq:gradient_boosting_step}
\phi_f \in \argmax_{\phi \in \mathcal{G}} \pd< \nabla_f \risk(f), \phi >_{L_2(\tpr_{\ndata}^X)}
\end{equation} 
and the gradient method in a function space is performed using a descent direction $-\phi_f$.
This approximation is a type of smoothing of functional gradients.
In kernel methods, this smoothing procedure is realized by using the {\it kernel smoothing} technique:
\begin{align} 
T_k \nabla_f \risk(f) 
\defeq \expec_{\tpr_{\ndata}^X}[ \nabla_f \risk(f)(X) k(X,\cdot)] 
 = \frac{1}{\ndata}\sum_{i=1}^{\ndata}\nabla_f \risk(f)(x_i) k(x_i,\cdot), \label{eq:kernel_smoothing}
\end{align}
where $k$ is a kernel function.
It should be noted that this kernel smoothing (\ref{eq:kernel_smoothing}) is a special case of gradient boosting (\ref{eq:gradient_boosting_step}) 
because of the following equation:
\begin{equation*}
\frac{T_k \nabla_f \risk(f)}{\| T_k \nabla_f \risk(f) \|_{\hilsp_k}} \in 
\argmax_{ \|\phi\|_{\hilsp_k} \leq 1 } \pd<\nabla_f \risk(f), \phi>_{L_2(\tpr_{\ndata}^X)}, \label{eq:characterize_kernel_smoothing}
\end{equation*}
where $(\hilsp_k, \pd<,>_{\hilsp_k})$ is the reproducing kernel Hilbert space associated with a kernel $k$.
When this kernel smoothing well approximates a functional gradient $\nabla_f\risk(f)$ and satisfies
\begin{equation}
\pd<\nabla_f \risk(f), T_k \nabla_f \risk(f)>_{L_2(\tpr_{\ndata}^X)} \geq \exists \mu \| \nabla_f \risk(f) \|_{L_2(\tpr_\ndata^X)}^2, \label{positivity_condition}
\end{equation}
the kernel-smoothed functional gradient descent $f^+ \leftarrow f - \eta T_k\nabla_f \risk(f)$ performs like the pure functional gradient descent, 
leading to the global convergence property because it tends to a stationary point in a function space, which is simply a global minimum.

Recently, several studies \citep{jacot2018neural,chizat2018note,du2018gradient,allen2018convergence,arora2019fine} 
implicitly or explicitly pointed out that the gradient descent for neural networks is essentially recognized as an approximation to the kernel-smoothed 
functional gradient method using a {\it neural tangent kernel} (NTK) \citep{jacot2018neural}: 
\begin{equation}
k_{NTK}(x,x') \defeq \expec_{\theta^{(0)} \sim \mu_0} [\partial_{\theta}\sigma(\theta^{(0)\top} x)^\top \partial_{\theta}\sigma(\theta^{(0)\top} x')], \label{eq:ntk}
\end{equation}
where $\mu_0$ is a distribution to initialize the parameters of the input layer in this setting.
In most proofs using NTK, the global convergence property has been demonstrated by showing the condition (\ref{positivity_condition}) from the positivity of the Gram-matrix $H^\infty \defeq (k_{NTK}(x_i,x_j))_{i,j=1}^\ndata$ and the similarity between the gradient descent and the kernel-smoothed functional gradient with NTK when $m \rightarrow \infty$.
This is a reason why very high over-parameterization is generally required in related studies. 

In this study, we found that the positivity of the Gram-matrix of NTK is not required on binary classification problems and a separability assumption, which is a weaker condition than the positivity, is enough for global convergence. 
Consequently, we can give global convergence and generalization guarantees to a gradient method for less over-parameterized two-layer neural networks.

\section{Global Convergence Analysis of the Gradient Method}
The following is an update rule of gradient descent with respect to the input parameters $\Theta=(\theta_r)_{r=1}^m$:
\begin{equation}\label{eq:gd}
\Theta^{(t+1)} \leftarrow \Theta^{(t)} - \eta \nabla_{\Theta} \risk(\Theta^{(t)}),    
\end{equation} 
where $\nabla_{\Theta} \risk(\Theta^{(t)}) = (\partial_{\theta_r} \risk(\Theta^{(t)}))_{r=1}^m$ and $\eta > 0$ is a learning rate.
We here make the assumption:
\begin{assumption}\label{assumption:convergence_analysis}
\item{{\bf(A1)}} Assume that $\mathrm{supp}(\tpr^X)\subset \{x\in \featuresp \mid\ \|x\|_2 \leq 1\}$. 
Let $\sigma$ be a $\mathcal{C}^2$-class function and there exist $K_1, K_2 > 0$ 
s.t. $\| \sigma'\|_\infty \leq K_1$ and $\|\sigma''\|_\infty \leq K_2$.
\item{{\bf(A2)}} A distribution $\mu_0$ on $\realsp^d$ used for the initialization of $\theta_r$ has a sub-Gaussian tail bound: $\exists A, \exists b > 0$ 
such that 
$\prob_{\theta^{(0)} \sim \mu_0}[\|\theta^{(0)}\|_2 \geq t] \leq A\exp(-bt^2)$.
\item{{\bf(A3)}}
Assume that the number of hidden units $m \in \posintegers$ is an even number.
Constant parameters $(a_r)_{r=1}^m$ and parameters $\Theta^{(0)}=(\theta_r^{(0)})_{r=1}^m$ are initialized symmetrically: 
$a_r=1$ for $r \in \{1,\ldots,\frac{m}{2}\}$, $a_r=-1$ for $r \in \{\frac{m}{2}+1,\ldots, m\}$, 
and $\theta_r^{(0)} = \theta_{r+\frac{m}{2}}^{(0)}$ for $r \in \{1,\ldots,\frac{m}{2}\}$, 
where the initial parameters $(\theta_r^{(0)})_{r=1}^{\frac{m}{2}}$ are independently drawn from a distribution $\mu_0$.
\item{{\bf(A4)}}
Assume that there exist $\rho > 0$ and a measurable function $v: \realsp^\fdim \rightarrow \{ w \in \realsp^\fdim \mid \|w\|_2\leq 1\}$ such that 
the following inequality holds: for $\forall (x,y) \in \mathrm{supp}(\tpr) \subset \featuresp \times \labelsp$, 
\begin{equation}
y \pd< \partial_\theta \sigma(\theta^{(0)\top} x), v(\theta^{(0)})>_{L_2(\mu_0)} = y\expec_{\theta^{(0)}\sim \mu_0}[\partial_\theta \sigma(\theta^{(0)\top} x)^\top v(\theta^{(0)}) ] \geq \rho. \label{eq:tangent_margin} 
\end{equation}
\end{assumption}
\paragraph{Remark.} 
Clearly, many activation functions (sigmoid, tanh, and smooth approximations of ReLU such as swish) satisfy the assumption {\bf(A1)}.
Typical distributions, including the Gaussian distribution, satisfy {\bf(A2)}.
The purpose of the symmetrized initialization {\bf(A3)} is to bound the initial value of the loss function $\risk(\Theta^{(0)})$ uniformly over the number of hidden units $m$.
This initialization leads to $f_{\Theta^{(0)}}(x)=0$, resulting in $\risk(\Theta^{(0)}) = \log(2)$.
Assumption {\bf(A4)} implies the separability of a dataset using the {\it neural tangent} model. We next discuss the validity of this assumption.

\subsection{Separability Assumption (A4) by the Neural Tangent}
The explicit feature representation: $x \rightarrow \partial_\theta \sigma(\theta^{(0)\top} x)$, of NTK (\ref{eq:ntk}) is called the {\it neural tangent}, which is a non-linear feature extraction from $\featuresp$ to an infinite-dimensional space.
That is, assumption {\bf(A4)} ensures the separability of the transformed data $(\partial_\theta \sigma(\theta_r^{(0)\top}x),y)$ through the neural tangent for $(x,y) \in \mathrm{supp}(\tpr)$ with a margin $\rho$ in an infinite-dimensional space by the weight: $v(\theta^{(0)})d\mu_0$.
We remark that this assumption is somewhat weaker than the positivity assumption on the Gram-matrix of NTK and is satisfied in many cases by the universal approximation ability of the neural tangent models.
In addition, we remark that the separability of the training dataset instead of $\mathrm{supp}(\tpr)$ is enough to guarantee global convergence only for empirical risk minimization. 

\paragraph{Theoretical comparison of kernel assumptions.} 
In previous studies \citep{du2018gradient,arora2019fine,zhang2019fast,wu2019global} the positivity of the Gram-matrix 
$H^\infty = (k_{NTK}(x_i,x_j))_{i,j=1}^\ndata$ was required to ensure the condition (\ref{positivity_condition}).
Here, we remark that the assumption {\bf(A4)} is weaker than this positivity condition in the following sense.
\begin{proposition} \label{prop:comparison_of_assumptions}
(i) Assume $H^\infty \succeq \lambda_0 I_\ndata$ and $\|\sigma'\|_\infty \leq K_1$, then there exists a measurable map $v: \realsp^\fdim \rightarrow \{ w \in \realsp^\fdim \mid \|w\|_2\leq 1\}$ such that $\forall i \in \{1,\ldots,\ndata\}$,
\begin{equation*} 
y_i \pd< \partial_\theta \sigma(\theta^{(0)\top} x_i), v(\theta^{(0)})>_{L_2(\mu_0)} \geq \frac{\lambda_0}{\ndata K_1}. 
\end{equation*}
(ii) Suppose assumption {\bf(A4)} holds, 
then $\sum_{i,j=1}^\ndata \xi_i H^\infty \xi_j \geq \rho^2 \|\xi\|_2^2$,  $(\forall \xi \in \{ (\alpha_i y_i)_{i=1}^\ndata \mid \alpha_i \geq 0 \})$.
\end{proposition}
As seen in Proposition \ref{prop:comparison_of_assumptions}-(i), the positivity $H^\infty \succeq \lambda_0 I_\ndata$ leads to weak separability with a margin of $O(\lambda_0/\ndata)$ on the training dataset.
Conversely, from Proposition \ref{prop:comparison_of_assumptions}-(ii), the separability with a margin of $\rho$ leads to the positivity $\rho^2$ of $H^\infty$ only on a cone spanned by the labels: $\{ (\alpha_i y_i)_{i=1}^\ndata \mid \alpha_i \geq 0 \}$.
Because this cone is very restrictive, this limited positivity is much weaker than the positivity on the whole space.
However, we found that this limited positivity is sufficient to ensure the global convergence of the gradient descent 
for binary classification problems with logistic loss.
Indeed, from Proposition \ref{prop:smoothness_org}, the positivity of $H^\infty$ is required only along the functional gradients: $\nabla_f \risk(f_\Theta)(x_i) = \partial_\zeta l(f_\Theta(x_i),y_i)$, and these functional gradients are always contained in this limited space, unlike the squared loss function.
This is a reason why the positivity of NTK is not required for the binary classification problem with logistic loss.
Thus, a much better convergence and generalization ability can be shown for logistic loss than the previous results that relied on the positivity of $H^\infty$ because the positivity of NTK on the whole space is redundant and a separability condition provides a better positivity only on a required small space.
Concretely, from Proposition \ref{prop:comparison_of_assumptions}-(i), we can immediately check a deteriorated convergence result depending on the positivity of $H^\infty$ by replacing $\rho$ in Theorem \ref{theorem:global_convergence} with $O(\lambda_0/\ndata)$, producing $m=O(\ndata \lambda_0^{-1}\epsilon^{-1})$ when $\beta=0$.

{\bf Remark.} For regression problem, \cite{allen2018convergence,allen2018convergence_b,zou2018stochastic,oymak2019towards,zou2019improved} make a different separation where examples are away from each other: $\|x_i - x_j\|_2 \geq \rho$.\cite{zou2019improved} shows that this assumption is essentially same as the positivity of NTK. Thus, it also completely differs from (A4) as shown in Proposition \ref{prop:comparison_of_assumptions}

\paragraph{Universal approximation property of neural tangent models.}
We consider the case where all feature vectors have a common bias term: $x=(x^0,\ldots,x^{d-1},s) \in \featuresp$ 
($s>0$ is a sufficiently small constant for a bias term).
In this case, we can easily confirm that neural tangent models include typical two-layer infinite-width neural networks with activation $\sigma'$: $\expec[ w(\theta^{(0)})\sigma'(\theta^{(0)^\top}x)]$ by setting $v(\theta^{(0)}) = (0,\ldots,0,w(\theta^{(0)}))$, where $w$ is a real-valued function.
Thus, Assumption (A.4) with a certainly positive constant $\rho$ is satisfied as long as a data distribution is separable by an infinite-width two-layer network with mild weights $w(\theta)$. 
Moreover, we note that these networks have the universal approximation property \citep{hornik1991approximation, sonoda2017neural}, so that there are a lot of examples such that the assumption {\bf(A4)} is satisfied.

\subsection{Main Results}\label{subsec:main_results}
We here define the $L_1$-norm of the functional gradient, which measures the convergence.
\begin{equation*}
    \| \nabla_f \risk(f_\Theta) \|_{L_{1}(\tpr_\ndata^X)} \defeq \frac{1}{\ndata}\sum_{i=1}^\ndata |\partial_\zeta l (f_\Theta(x_i),y_i)| = \frac{1}{2\ndata} \sum_{i=1}^\ndata | y_i - 2 p_{\Theta}(Y=1\mid x_i) + 1|. 
\end{equation*}
Here, $\zeta$ is the first variable of $l$ and $p_{\Theta}(Y=1\mid x) = \frac{1}{1 + \exp(- f_{\Theta}(x))}$ is a conditional probability on $Y=1$, defined by $f_{\Theta}$.
Because $\| \nabla_f \risk(f_\Theta) \|_{L_{1}(\tpr_\ndata^X)}$ is simply a gap between the labels and conditional label probabilities of the model, 
this norm is a reasonable measure for the binary classification problems. 
The following is the first main result to ensure global convergence.
\begin{theorem}[Global Convergence]\label{theorem:global_convergence}
Suppose Assumption \ref{assumption:convergence_analysis} holds.
Set $K = K_1^4 + 2K_1^2K_2 + K_1^4K_2^2$.
For $\forall \beta \in [0,1)$, $\forall \delta \in (0,1)$ and $\forall m \in \posintegers$ such that $m \geq \frac{16K_1^2}{\rho^2}\log \frac{2\ndata}{\delta}$, 
consider gradient descent (\ref{eq:gd}) with 
a learning rate of $0< \eta \leq \min\left\{ m^\beta, \frac{4m^{2\beta-1}}{K_1^2 + K_2} \right\}$ 
and the number of iterations $T \leq \floor[\frac{m \rho^2}{32\eta K_2^2 \log(2)}]$.
Then, it follows that with probability at least $1-\delta$ over the random initialization,
\begin{equation}
\frac{1}{T}\sum_{t=0}^{T-1}\| \nabla_f \risk(f_{\Theta^{(t)}}) \|_{L_{1}(\tpr_\ndata^X)}^2 
\leq \frac{16\log(2)}{\rho^2 T} \left( \frac{m^{2\beta-1}}{\eta} + K \right).
\end{equation}
\end{theorem}
We here derive a corollary, which states that an arbitrary small empirical classification error can be achievable by appropriately setting $\eta, m,$ and $T$ as follows.
From Markov's inequality,
\begin{align*}
\mathbb{P}_{(X,Y)\sim \tpr_\ndata}[Yf_{\Theta^{(t)}}(X)\leq 0]
\leq 2 \| \nabla_f \risk(f_\Theta) \|_{L_1(\tpr_\ndata^X)}, 
\end{align*}
where we used the following relationship:
\[ 0.5 \left| y_i - 2p_{\Theta}(Y=1|x_i) + 1 \right| \geq (1+\exp(\gamma))^{-1} 
\Longleftrightarrow y_i f_\Theta(x_i) \leq \gamma. \]
Thus, a convergence rate of $\| \nabla_f \risk(f_{\Theta^{(t)}}) \|_{L_{1}(\tpr_\ndata^X)}^2$ leads to a rate of empirical classification error.
\begin{corollary} \label{corollary:convergence_rate}
Suppose the same assumptions as in Theorem \ref{theorem:global_convergence} hold.
If for $\forall \epsilon, \delta > 0$, the hyperparameters satisfy
\[ \beta \in [0,1),\ 
m=\Omega(\rho^{\frac{-1}{1-\beta}}\epsilon^{\frac{-1}{1-\beta}}),\ 
T=\Omega(\rho^{-2}\epsilon^{-2}),\ 
\eta=\Theta(m^{2\beta-1}), \]
then with probability at least $1-\delta$, gradient descent (\ref{eq:gd}) with a learning rate of $\eta$ finds
a parameter $\Theta^{(t)}$ satisfying $\mathbb{P}_{(X,Y)\sim \tpr_\ndata}[Yf_{\Theta^{(t)}}(X)\leq 0] \leq \epsilon$ within $T$-iterations.
\end{corollary}
The Landau notations are applied with respect to $\epsilon, \rho \rightarrow 0$.
Utilizing this theorem, we can show the convergence of the loss function which leads to a better result at the price of slight increase of $m$.
\begin{theorem} \label{theorem:faster_global_convergence}
Suppose the same assumptions in Theorem \ref{theorem:global_convergence} hold. 
Then there exists a uniform constant $C>0$ such that for $0 < \forall \alpha \leq \frac{\rho}{4K_2}$, with probability at least $1-\delta$, 
\begin{align*}
\frac{1}{T}\sum_{t=0}^{T-1} \risk(\Theta^{(t)}) 
\leq C\left( \frac{1}{T} + \frac{\alpha^2 m}{\eta T} + \exp \left(- \frac{\alpha \rho m^{1-\beta}}{4}\right) 
+ \frac{\alpha^2}{\rho} \sqrt{\frac{m}{\eta T}}
+ \frac{1}{\rho} \sqrt{\frac{\eta T}{m}} \right).
\end{align*}
\end{theorem}

\begin{corollary} \label{corollary:faster_convergence_rate}
Suppose the same assumptions as in Theorem \ref{theorem:faster_global_convergence} hold.
Then there exists a uniform constant $C>0$ and the following statement holds. 
If for any $\epsilon > 0$, the hyperparameters satisfy
\[ \beta = 0,\
m=\Theta(\rho^{-2}\epsilon^{-3/2} \log(1/\epsilon)),\ 
T_*=\Theta(\rho^{-2}\epsilon^{-1}\log^2(1/\epsilon)),\ 
\eta=\Theta(m^{-1}),\  \]
then with probability $1-\delta$, $\frac{1}{T} \sum_{t=0}^{T-1} \risk(\Theta^{(t)}) \leq C \left( \epsilon + \frac{1}{\rho^2 T} \log^2\left(\frac{1}{\epsilon}\right)\right)$ for $0 < \forall T \leq T_*$.
\end{corollary}

This corollary is obtained immediately from Theorem \ref{theorem:faster_global_convergence} 
by setting $\alpha = \Theta(\rho\epsilon^{3/2})$.
The convergence of classification error also derived from the corollary because $\risk(f_\theta) \geq \| \nabla_f \risk(f_{\Theta^{(t)}}) \|_{L_{1}(\tpr_\ndata^X)}$ and it also derives a sharper bound on the distance $\|\Theta^{(T)}-  \Theta^{(0)}\|_2$.

\begin{proposition} \label{proposition:sharper_bound_on_distance}
Suppose the same assumption and consider the same hyperparameter setting as in Corollary \ref{corollary:faster_convergence_rate}.
Then, there exists a uniform constant $C$ such that with probability $1-\delta$, 
\[ \|\Theta^{(T)}-  \Theta^{(0)}\|_2 \leq C\epsilon^{3/4} \log^2(\rho^{-2}\epsilon^{-1}). \]
\end{proposition}  

Moreover, by combining Theorem \ref{theorem:global_convergence} and Corollary \ref{corollary:faster_convergence_rate} with the well-known margin bound \citep{kp2002,mohri2012foundations,shalev2014understanding} on the expected classification error and by specifying the Rademacher complexity of the function class attained by gradient descent,
we can obtain the second main result for the generalization ability of gradient descent.
\begin{theorem}[Generalization Bound]\label{theorem:generalization_bound}
Suppose Assumption \ref{assumption:convergence_analysis} holds.
Set $K = K_1^4 + 2K_1^2K_2 + K_1^4K_2^2$.
Fix $\forall \gamma > 0$.
Consider the gradient descent (\ref{eq:gd}) with a general hyperparameter setting in Theorem \ref{theorem:global_convergence} or a specific setting in Corollary \ref{corollary:faster_convergence_rate}, with $\delta \in (0,1)$.
For these cases, we set parameters $C_{\eta,m,T}$ and $D_{\eta,m,T}$ as follows:
\begin{align*}
\textrm{(Former case)} \hspace{5mm}  C_{\eta,m,T} &= \rho^{-1} T^{-1/2}\left( m^{\beta-\frac{1}{2}}\eta^{-1/2} + \sqrt{K} \right),\hspace{5mm} D_{\eta,m,T} = \sqrt{\eta T}  \\
\textrm{(Latter case)} \hspace{5mm}  C_{\eta,m,T} &= \epsilon + \rho^{-2}T^{-1} \log^2\left(1/\epsilon\right), \hspace{5mm} D_{\eta,m,T} = \epsilon^{3/4} \log^2(\rho^{-2}\epsilon^{-1}).
\end{align*}
Then, there exists a uniform constant $C>0$ and it follows that with probability at least $1-3\delta$ 
over a random initialization and random choice of dataset $S$,
\begin{align}
&\hspace{-5mm}\min_{t \in \{0,\ldots,T-1\}} \mathbb{P}_{(X,Y)\sim \tpr}[Yf_{\Theta^{(t)}}(X)\leq 0] 
\leq C(1+\exp(\gamma))C_{\eta,m,T} + 
3\sqrt{\frac{ \log(2/\delta)}{2\ndata}} \notag \\
&+ C\gamma^{-1}m^{\frac{1}{2}-\beta}D_{\eta,m,T}(1+K_1+K_2)\sqrt{\frac{d}{\ndata} 
\log\left( \ndata(1+K_1+K_2)(\log(m/\delta)+ D_{\eta,m,T}^2) \right) }. \label{generalization_bound}
\end{align}
Moreover, when $\sigma$ is convex and $\sigma(0) = 0$, we can avoid the dependence with respect to the dimension $d$. With probability at least $1-3\delta$ 
over a random initialization and random choice of dataset $S$
\begin{align}
&\hspace{-5mm}\min_{t \in \{0,\ldots,T-1\}} \mathbb{P}_{(X,Y)\sim \tpr}[Yf_{\Theta^{(t)}}(X)\leq 0] 
\leq C(1+\exp(\gamma))C_{\eta,m,T} + 
3\sqrt{\frac{ \log(2/\delta)}{2\ndata}} \notag \\
&+ \frac{CK_1 m^{\frac{1}{2}-\beta}}{\gamma\sqrt{n}}\bigg(D_{\eta,m,T}+ \sqrt{\frac{\log(Am/\delta)}{b} } \bigg). \label{generalization_bound_2}
\end{align}
\end{theorem}
This theorem provides an upper-bound on an expected classification error with high probability for a network obtained by gradient descent within $T$-iterations.
There is a trade-off between the optimization and complexity terms in (\ref{generalization_bound}), (\ref{generalization_bound_2}) with respect to $\eta, m,$ and $T$. However, there are several choices of these hyperparameters to achieve a desired precision $\epsilon$ of the expected classification error. 
\begin{corollary} \label{corollary:error_convergence}
Suppose the same assumptions as in Theorem \ref{theorem:generalization_bound} hold.
If for any $\epsilon > 0$, the hyperparameters satisfy one of the following
\begin{align*}
&\textrm{(i)}\  \beta \in [0,1),\ 
m=\Omega(\rho^{\frac{-2}{1-\beta}}\epsilon^{\frac{-1}{1-\beta}}),\ 
T=\Omega(\rho^{-2}\epsilon^{-2}),\ 
\eta=\Theta(\rho^{-2}\epsilon^{-2}T^{-1}m^{2\beta-1}),\ 
  \ndata = \tilde{\Omega}(\rho^{-2}\epsilon^{-4}), \\
&\textrm{(ii)}\  \beta = 0 ,\ 
m=\Theta(\rho^{-2}\epsilon^{-3/2}\log(1/\epsilon)),\ 
T=\Theta(\rho^{-2}\epsilon^{-1}\log^2(1/\epsilon)),\ 
\eta=\Theta(m^{-1}),\ 
  \ndata = \tilde{\Omega}(\epsilon^{-2}),  
\end{align*}
then with probability at least $1-\delta$, the gradient descent (\ref{eq:gd}) with a learning rate of $\eta$ finds
a parameter $\Theta^{(t)}$ satisfying $\mathbb{P}_{(X,Y)\sim \tpr}[Yf_{\Theta^{(t)}}(X)\leq 0] \leq \epsilon$ within $T$-iterations.
\end{corollary}
This corollary can be immediately proven by substituting the concrete values of $\beta, m, T, \eta,$ and $\ndata$ into the right hand side of inequalities (\ref{generalization_bound}), (\ref{generalization_bound_2}) and by checking that this hyperparameter setting satisfies the conditions required in Theorem \ref{theorem:generalization_bound}.

From Corollary \ref{corollary:error_convergence}, for an arbitrary small $\epsilon>0$, an expected $\epsilon$-classification error is achieved by the gradient descent 
within $O(1/\epsilon^2)$-or $O(1/\epsilon)$-iterations when the transformed data distribution by the neural tangent: 
$(\partial_{\theta} (\theta_r^{(0)\top} \cdot))_{\theta \sim \mu_0}$ is separable in the infinite-dimensional space $L_2(\mu_0)$ 
under the $L_\infty$-constraint with a sufficient margin $\rho$. 
In comparison to the results in \cite{allen2018learning,cao2019generalization,cao2019generalization_b}, which also derived a generalization bound by making a similar separability assumption
using a ReLU network or a smooth target function instead of the tangent model, our result has much better dependency on the network width and can explain the generalization ability for a less over-parameterized two-layer network, as summarized in Table \ref{table:comparison_of_hyperparameters_settings}.
However, we note that their theories cover deeper networks and are not included in our theory because of the difference of the problem settings (e.g., network depth and the type of activation functions).
To reduce the network width, the best choice of $\beta \in [0,1)$ is $\beta=0$ for the first setting, leading to a small network width $m=\Omega(\epsilon^{-1})$.
We note that an arbitrary large width is also covered by this result.

\subsection{Proof Idea}\label{subsec:proof_idea}
In this section, we provide a proof idea for Theorems \ref{theorem:global_convergence}, \ref{theorem:faster_global_convergence}, and \ref{theorem:generalization_bound}. 
We introduce two important propositions that connect the gradient descent with the functional gradient descent to justify an intuitive explanation in Section \ref{sec:overview}. 
Proposition \ref{prop:smoothness_org}-(i) states that the gradient descent method is certainly similar to the kernel smoothed gradient methods by the neural tangent kernel 
when a parameter $\Theta$ is sufficiently close to a stationary point and the learning rate $\eta$ is sufficiently small, and 
Proposition \ref{prop:smoothness_org}-(ii) states that the loss landscape is almost convex with respect to the parameter when $f_\Theta$ is sufficiently close to a stationary point in the function space.
Here, We define an approximated neural tangent kernel: $k_\Theta$ depending on the parameters $\Theta$ as follows:
\begin{equation}
k_\Theta(x,x') \defeq \partial_\Theta f_\Theta(x)^\top \partial_\Theta f_\Theta(x'). \label{eq:aprox_ntk}
\end{equation}
This kernel is actually an approximation to the vanilla NTK as follows:
\[ m^{2\beta-1}k_{\Theta^{(0)}}(x,x') \rightarrow k_{NTK}(x,x') \ \ (m\rightarrow \infty). \]

\begin{proposition}\label{prop:smoothness_org}
Suppose assumption {\bf(A1)} holds and $\beta \in [0,1)$.\\
(i) We set $\Theta^+ = \Theta - \eta \nabla_\Theta \risk(\Theta)$ and $K = K_1^2 + 2K_2 + K_1^2K_2^2$.
If $\eta \leq m^\beta$, then
\begin{align*}
\Bigl| \risk(f_{\Theta^+}) - 
\left( \risk(f_{\Theta}) -  \eta 
\pd< \nabla_f \risk(f_{\Theta}), T_{k_{\Theta}} \nabla_f \risk(f_{\Theta})>_{L_2(\tpr_\ndata^X)} \right) \Bigr| 
\leq \frac{\eta^2 K}{2m^{2\beta-1}} \| \nabla_\Theta \risk(\Theta) \|_2^2.
\end{align*}
(ii) It follows that for $\Theta = (\theta_r)_{r=1}^m$ and $\Theta^* = (\theta_r^*)_{r=1}^m$, $(\theta_r, \theta_r^* \in \realsp^\fdim)$,
\begin{align*}
\risk(\Theta) + \nabla_\Theta \risk(\Theta)^\top (\Theta^* - \Theta) 
\leq 
\risk(\Theta^*) + 
\frac{K_2}{m^\beta}\|\nabla_f\risk(f_{\Theta})\|_{L_1(\tpr_\ndata^X)} \|\Theta^* - \Theta\|_2^2. 
\end{align*}
\end{proposition}

The next proposition states that the kernel smoothed gradients have comparable optimization ability to pure functional gradients 
in terms of minimizing the $L_1$-norm around an initial parameter $\Theta^{(0)}$.
We define the $\|\cdot\|_{2,1}$-norm in the parameter space $\Theta=(\theta_r)_{r=1}^m$ as $\|\Theta\|_{2,1}\defeq \sum_{r=1}^{m}\| \theta_r \|_2$.
\begin{proposition}\label{prop:kernel_positivity_org}
Suppose Assumption \ref{assumption:convergence_analysis} holds.
For $\forall \delta \in (0,1)$ and $\forall m \in \posintegers$, such that $m \geq \frac{16K_1^2}{\rho^2}\log \frac{2\ndata}{\delta}$,
the following statement holds with probability at least $1-\delta$ over the random initialization of $\Theta^{(0)}=(\theta_r^{(0)})_{r=1}^m$.
If $\|\Theta - \Theta^{(0)}\|_{2,1} \leq \frac{m\rho}{4K_2}$, then
\begin{equation*}
\pd< \nabla_f \risk(f_{\Theta}), T_{k_{\Theta}}\nabla_f \risk(f_\Theta)>_{L_2(\tpr_\ndata^X)}  
\geq \frac{\rho^2}{ 16m^{2\beta-1}}  \| \nabla_f \risk(f_\Theta) \|_{L_{1}(\tpr_\ndata^X)}^2.
\end{equation*}
\end{proposition}
This proposition is specialized to binary classification problems because the positivity of the Gram-matrix is required for regression problems to  
make a similar statement as discussed earlier.

Combining these two propositions, we can connect the gradient descent with the functional gradient descent and show the global convergence (Theorem \ref{theorem:global_convergence}) which with the almost convexity (Proposition \ref{prop:smoothness_org}-(ii)) is also used to derive the convergence rate for the loss function (Theorem \ref{theorem:faster_global_convergence}).
In addition, applying a well-known result \citep{kp2002}, we can derive a skeleton of a generalization bound (Theorem \ref{theorem:generalization_bound}) composed of the margin distribution and Rademacher complexity.
As for the margin distribution, the upper bound is obtained by Theorem \ref{theorem:global_convergence} and \ref{theorem:faster_global_convergence}.
As for ways to bound Rademacher complexity, please see the Appendix.

\section{Conclusion}\label{sec:conclusion}
In this paper, we have provided refined global convergence and generalization analyses of the gradient descent for 
two-layer neural networks with smooth activations on binary classification problems.
The key in our analysis is the separability assumption by a neural tangent model and we have explained the reasonability of this assumption in comparison to the positivity of NTK.
Consequently, theoretical justification has been provided for less over-parameterized neural networks.
However, our theory is restricted to the {\it deterministic} gradient descent and two-layer networks; hence, 
its possible extensions to stochastic gradient descent and deep neural networks are also interesting.
Another possible future study is to relax the positivity assumption on the Gram-matrix for regression problems by utilizing our theory and 
conducting further investigations of the trajectory of gradient descent, such as the shortest pass analysis \citep{oymak2018overparameterized}.

%% file: supplement_body.tex
\section{Relationship between Kernel Assumptions}
\begin{proof}[Proof of Proposition \ref{prop:comparison_of_assumptions}]
We here prove the statement (i).
Since $H^\infty$ is invertible, we set $w = (H^{\infty})^{-1}(y_1, \cdots, y_n)^\top$ and set 
\[ v(\theta^{(0)}) = \sum_{j=1}^\ndata \partial_\theta \sigma(\theta^{(0)\top}x_j) w_j. \]
Then, we get
\[ y_i \pd< \partial_\theta \sigma(\theta^{(0)\top} x_i), v(\theta^{(0)})>_{L_2(\mu_0)} = y_i H^\infty_{i*}w = 1. \]
We can bound the norm of $\|v(\theta^{(0)})\|_2$ as follows:
\begin{align*}
 \|v(\theta^{(0)})\|_2  
 &\leq \sum_{j=1}^\ndata \| \partial_\theta \sigma(\theta^{(0)\top}x_j) \|_2 |w_j| \\
 &\leq \left\| \left( \| \partial_\theta \sigma(\theta^{(0)\top}x_j) \|_2 \right )_{j=1}^\ndata \right\|_2 \|w\|_2 \\ 
 &\leq \sqrt{\ndata} K_1 \|w\|_2 \\
 &\leq \frac{\ndata K_1}{\lambda_0}. 
\end{align*}
Thus, by resetting $v(\theta^{(0)}) \leftarrow \frac{\lambda_0 v(\theta^{(0)})}{\ndata K_1}$, we conclude the statement (i).

We next prove the statement (ii).
For $\xi = (\alpha_i y_i)_{i=1}^\ndata$ ($\alpha_i>0$), 
\begin{align*}
\sum_{i,j=1}^\ndata \xi_i H^\infty \xi_j
&= \sum_{i,j=1}^\ndata \expec_{\theta^{(0)}\sim\mu_0}[ \xi_i \partial_\theta( \theta^{(0)\top}x_i)^\top \partial_\theta( \theta^{(0)\top}x_j ) \xi_j ] \\
&= \expec_{\theta^{(0)}\sim\mu_0} \left [ \left\| \sum_{i=1}^\ndata \xi_i \partial_\theta( \theta^{(0)\top}x_i ) \right \|_2^2 \right ] \\
& \geq \expec_{\theta^{(0)}\sim\mu_0} \left [ \left( \sum_{i=1}^\ndata \xi_i \partial_\theta( \theta^{(0)\top}x_i )^\top v(\theta^{(0)}) \right)^2  \right ] \\
& \geq \left( \expec_{\theta^{(0)}\sim\mu_0} \left [ \sum_{i=1}^\ndata \xi_i \partial_\theta( \theta^{(0)\top}x_i )^\top v(\theta^{(0)}) \right ] \right)^2  \\
& = \left( \sum_{i=1}^\ndata \alpha_i \expec_{\theta^{(0)}\sim\mu_0} \left [ y_i \partial_\theta( \theta^{(0)\top}x_i )^\top v(\theta^{(0)}) \right ] \right)^2  \\
& \geq \rho^2 \left(\sum_{i=1}^\ndata \alpha_i \right)^2 \\
& \geq \rho^2 \sum_{i=1}^\ndata \alpha_i^2 \\
& = \rho^2 \|\xi\|_2^2,
\end{align*}
where we used $\|v(\theta^{(0)})\|_2 \leq 1$ for the first inequality, the convexity of $\|\cdot\|_2^2$ and Jensen's inequality for the second inequality, Assumption {\bf(A4)} for the third inequality, and $\| \cdot \|_2 \leq \| \cdot \|_1$ for the last inequality.
Thus, we finish the proof of the statement (ii).
\end{proof}

\section{Auxiliary Results}\label{sec:auxiliary_results}
In this section, we introduce several existing results for proving our statements.
We first describe the Hoeffding's inequality.

\begin{lemma}[Hoeffding's inequality] \label{lemma:hoeffding_lemma}
Let $Z, Z_1,\ldots,Z_m$ be i.i.d. random variables taking values in $[-a,a]$ for $a>0$.
Then, for any $\epsilon > 0$, we get
  \begin{equation*}
    \prob\left[ \left|\frac{1}{m}\sum_{r=1}^m Z_r - \expec[Z] \right| > \epsilon \right] 
    \leq 2\exp\left( - \frac{\epsilon^2 m}{2a^2}\right). 
  \end{equation*} 
\end{lemma}

We here define the {\it covering number} as follows. 
\begin{definition}[Covering Number]
Let $(V,\|\cdot\|)$ a metric space.
A subset $\hat{U} \subset V$ is called an $\epsilon$-(proper) cover of $V$ 
if for $\forall v \in V$, there exists $v' \in \hat{U}$ such that $\|v-v'\| < \epsilon$.
Then, $\epsilon$-covering number $\mathcal{N}(V, \epsilon, \|\cdot\|)$ of $V$ is defined as the cardinally of the smallest $\epsilon$-cover of $V$, that is,
\[ \mathcal{N}(V, \epsilon, \|\cdot\|) \defeq \min\{ |\hat{U}| \mid \textrm{ $\hat{U}$ is an $\epsilon$-cover of $V$} \}. \]
\end{definition}

The following lemma provide a bound on the Rademacher complexity by Dudley's integral.
For a real-valued function class $\mathcal{F}$ over $\featuresp$ and a subset $X=(x_i)_{i=1}^\ndata$, 
$\mathcal{F}|_X$ is defined as $\{ (h(x_i))_{i=1}^\ndata \in \realsp^\ndata \mid h \in \mathcal{F} \} \subset \realsp^\ndata$, 
and $\mathcal{F}|_X$ can be equipped with $\|\cdot\|_\infty$-norm over $X$.

\begin{lemma}[\cite{bartlett2017spectrally}] \label{lemma:dudley_integral}
Let $\mathcal{F}$ be a class of real-valued functions taking values in $[0,1]$ from $\featuresp$ and assume $0 \in \mathcal{F}$.
For examples $\forall X=(x_i)_{i=1}^\ndata$ of size $\ndata$, we get 
\[ \radcomp(\mathcal{F}|_X) \leq 
\inf_{\alpha>0} \left( 4\alpha
+ \frac{12}{\sqrt{\ndata}}\int_{\alpha}^{1}\sqrt{\log(\mathcal{N}(\mathcal{F}|_X,\epsilon,\|\cdot\|_\infty))}d\epsilon \right). \]
\end{lemma}
Note that we reformulate the statement in Lemma \ref{lemma:dudley_integral} from $\|\cdot\|_2$-covering to $\|\cdot\|_\infty$-covering.

\section{Proofs of Main Results}
In this section, we give an outline of proofs of Theorem \ref{theorem:global_convergence} and \ref{theorem:generalization_bound}. 
\paragraph{Global convergence.}
We first introduce two important propositions which connects gradient methods with functional gradient methods. 
The following proposition states that gradient descent methods become similar to kernel smoothed gradient methods by the neural tangent kernel 
when a parameter $\Theta$ is sufficiently close to a stationary point and a learning rate $\eta$ is sufficiently small.
\begin{proposition}[Restatement of Proposition \ref{prop:smoothness_org}] \label{prop:smoothness}
Suppose assumption {\bf(A1)} holds and $\beta \in [0,1)$.\\
(i) We set $\Theta^+ = \Theta - \eta \nabla_\Theta \risk(\Theta)$ and $K = K_1^2 + 2K_2 + K_1^2K_2^2$.
If $\eta \leq m^\beta$, then
\begin{align*}
\Bigl| \risk(f_{\Theta^+}) - 
\left( \risk(f_{\Theta}) -  \eta 
\pd< \nabla_f \risk(f_{\Theta}), T_{k_{\Theta}} \nabla_f \risk(f_{\Theta})>_{L_2(\tpr_\ndata^X)} \right) \Bigr| 
\leq \frac{\eta^2 K}{2m^{2\beta-1}} \| \nabla_\Theta \risk(\Theta) \|_2^2.
\end{align*}
(ii) It follows that for $\Theta = (\theta_r)_{r=1}^m$ and $\Theta^* = (\theta_r^*)_{r=1}^m$, $(\theta_r, \theta_r^* \in \realsp^\fdim)$,
\begin{align*}
\risk(\Theta) + \nabla_\Theta \risk(\Theta)^\top (\Theta^* - \Theta) 
\leq 
\risk(\Theta^*) + 
\frac{K_2}{m^\beta}\|\nabla_f\risk(f_{\Theta})\|_{L_1(\tpr_\ndata^X)} \|\Theta^* - \Theta\|_2^2. 
\end{align*}
\end{proposition}

The next proposition states that kernel smoothed gradients have comparable optimization ability to pure functional gradients in terms of the $L_1$-norm 
around an initial parameter $\Theta^{(0)}$.
We introduce the $\|\cdot\|_{2,1}$-norm in the parameter space $\Theta=(\theta_r)_{r=1}^m$ 
as $\|\Theta\|_{2,1}\defeq \sum_{r=1}^{m}\| \theta_r \|_2$.

\begin{proposition}[Restatement of Proposition \ref{prop:kernel_positivity_org}] \label{prop:kernel_positivity}
Suppose Assumption \ref{assumption:convergence_analysis} holds.
For $\forall \delta \in (0,1)$ and $\forall m \in \posintegers$, such that $m \geq \frac{16K_1^2}{\rho^2}\log \frac{2\ndata}{\delta}$,
the following statement holds with probability at least $1-\delta$ over the random initialization of $\Theta^{(0)}=(\theta_r^{(0)})_{r=1}^m$.
If $\|\Theta - \Theta^{(0)}\|_{2,1} \leq \frac{m\rho}{4K_2}$, then
\begin{equation*}
\pd< \nabla_f \risk(f_{\Theta}), T_{k_{\Theta}}\nabla_f \risk(f_\Theta)>_{L_2(\tpr_\ndata^X)}  
\geq \frac{\rho^2}{ 16m^{2\beta-1}}  \| \nabla_f \risk(f_\Theta) \|_{L_{1}(\tpr_\ndata^X)}^2.
\end{equation*}
\end{proposition}
This proposition is specialized to binary classification problems because the positivity of the Gram-matrix is needed for regression problems in order to  
make a similar statement as discussed earlier.

We specify the possible number of iterations of gradient descent (\ref{eq:gd}) such that $\Theta^{(t)}$ can remain 
in the neighborhood: $\{\Theta \mid \|\Theta-\Theta^{(0)}\|_{2} \leq \frac{\sqrt{m}\rho}{4K_2} \} 
\subset \{\Theta \mid \|\Theta-\Theta^{(0)}\|_{2,1} \leq \frac{m\rho}{4K_2} \}$.
\begin{proposition}\label{prop:gd_iterations}
Suppose Assumption {\bf(A1)} and {\bf(A3)} hold.
Consider gradient descent (\ref{eq:gd}) with learning rate $0 < \eta < \frac{4m^{2\beta-1}}{K_1^2 + K_2}$ and 
the number of iterations $T\in \posintegers$.
Then, 
\begin{equation}
\frac{1}{T}\sum_{t=0}^{T-1}\| \nabla_\Theta \risk(\Theta^{(t)}) \|_2^2  
\leq \frac{2}{\eta T}\log(2). \label{eq:gd_local_conv_rate}
\end{equation}  
Especially, we get $\| \Theta^{(T)} - \Theta^{(0)}\|_{2} \leq \sqrt{2 \eta T \log(2)}$. 
As a result, gradient descent can be performed for $\floor[\frac{m\rho^2}{32\eta K_2^2 \log(2)}]$-iterations 
within $\{\Theta \mid \|\Theta-\Theta^{(0)}\|_{2} \leq \frac{\sqrt{m}\rho}{4K_2} \} 
\subset \{\Theta \mid \|\Theta-\Theta^{(0)}\|_{2,1} \leq \frac{m\rho}{4K_2} \}$.
\end{proposition}
This proposition provides a bound on the distance $\|\Theta^{(T)}- \Theta^{(0)}\|_2$, but we note that this bound will be further sharpened after showing the convergence of the loss function (see Proposition \ref{proposition:sharper_bound_on_distance}).
From Proposition \ref{prop:smoothness}, \ref{prop:kernel_positivity}, and \ref{prop:gd_iterations}, we notice that the gradient descent for $\risk(\Theta)$ 
performs like a pure functional gradient descent up to $O\left(\frac{m\rho^2}{\eta}\right)$-iterations, 
resulting in significant decrease of loss functions.
We next provide the proof of Theorem \ref{theorem:global_convergence} based on this idea.

\begin{proof}[Proof of Theorem \ref{theorem:global_convergence}]
From Proposition \ref{prop:gd_iterations}, the assumption in Proposition \ref{prop:kernel_positivity} regarding $\Theta$ is satisfied.
Thus, Proposition \ref{prop:smoothness} and \ref{prop:kernel_positivity} state that for $t \in \{0,\ldots,T-1\}$, 
\begin{align*}
\risk(f_{\Theta^{(t+1)}}) 
\leq \risk(f_{\Theta^{(t)}}) 
- \frac{\eta\rho^2}{16 m^{2\beta-1}} \| \nabla_f \risk(f_{\Theta^{(t)}}) \|_{L_{1}(\tpr_\ndata^X)}^2 
+ \frac{\eta^2 K}{2 m^{2\beta-1}}\| \nabla_\Theta \risk(\Theta^{(t)})\|_2^2.
\end{align*}
Summing this inequality over $t \in \{0,\ldots,T-1\}$ and multiplying by $\frac{4 m^{2\beta-1}}{\eta \rho^2 T}$, we have
\begin{align*}
\frac{1}{T}\sum_{t=0}^{T-1}\| \nabla_f \risk(f_{\Theta^{(t)}}) \|_{L_{1}(\tpr_\ndata^X)}^2
&\leq \frac{16 m^{2\beta-1}}{\eta \rho^2 T} \risk(f_{\Theta^{(0)}}) 
+ \frac{8\eta K}{\rho^2 T} \sum_{t=0}^{T-1}\| \nabla_\Theta \risk(\Theta^{(t)})\|_2^2. 
\end{align*}
Applying $\risk(\Theta^{(0)})=\log(2)$ and inequality (\ref{eq:gd_local_conv_rate}), we complete the proof.
\end{proof}

\begin{proof}[Proof of Theorem \ref{theorem:faster_global_convergence}]
We set $\tau^* = \left(\alpha a_r v(\theta_r^{(0)}) \right)_{r=1}^m$ and $\Theta^* = \Theta^{(0)} + \tau^*$.
Clearly, we have
\begin{equation}
\| \Theta^* - \Theta^{(0)}\|_2 \leq \alpha \sqrt{m}. \label{eq:reference_distance}
\end{equation}
As shown in Proposition \ref{prop:smoothness}, we get 
\[ \left| f_{\Theta^*}(x) - \nabla_\Theta f_{\Theta^{(0)}}(x)^\top (\Theta^* - \Theta^{(0)}) \right| \leq \frac{K_2}{m^\beta} \|\tau^* \|_2^2 \leq \alpha^2 K_2 m^{1-\beta}. \]
In addition, as shown in Proposition \ref{prop:kernel_positivity}, since $m \geq \frac{16 K_1^2}{\rho^2}\log \frac{2\ndata}{\delta}$, the inequality (\ref{eq:high_prob_separability}) holds with probability at least $1-\delta$.
Hence, we have for $\forall i \in \{1,\ldots, \ndata\}$,
\begin{align*}
y_i f_{\Theta^*}(x_i) 
&\geq y_i \nabla_\Theta f_{\Theta^{(0)}}(x)^\top (\Theta^* - \Theta^{(0)}) - \alpha^2 K_2 m^{1-\beta} \notag \\
&= \frac{y_i \alpha}{m^\beta} \sum_{r=1}^m \partial_\theta \sigma (\theta_r^{(0)\top}x_i)^\top v(\theta_r^{(0)}) - \alpha^2 K_2 m^{1-\beta} \notag \\
&\geq \frac{\alpha \rho m^{1-\beta}}{2} - \alpha^2 K_2 m^{1-\beta} \geq \frac{\alpha \rho m^{1-\beta}}{4}.  \notag
\end{align*}
Thus, the loss at a reference point $\Theta^*$ can be bounded as follows:
\begin{align}
\risk(\Theta^*) 
\leq \frac{1}{\ndata}\sum_{i=1}^\ndata \log\left( 1 + \exp(-y_i f_{\Theta^*}(x_i) \right) 
\leq \exp\left(-y_i f_{\Theta^*}(x_i) \right) 
\leq \exp\left( -\frac{\alpha \rho m^{1-\beta}}{4} \right). \label{eq:reference_risk_bound}
\end{align}

From Theorem \ref{theorem:global_convergence}, Proposition \ref{prop:smoothness}-(ii), Proposition \ref{prop:gd_iterations} and inequalities (\ref{eq:reference_distance}), (\ref{eq:reference_risk_bound}), 
it follows that $\exists C_1, \exists C_2 > 0$, $\forall T \leq T$,
\begin{align}
\frac{1}{T} \sum_{t=0}^{T-1} &\left( \risk(\Theta^{(t)}) + \nabla_\Theta \risk (\Theta^{(t)})^\top (\Theta^* - \Theta^{(t)})\right) \notag \\
&\leq \risk(\Theta^*) + \frac{K_2}{m^\beta T}\sum_{t=0}^{T-1}\|\nabla_f\risk(f_{\Theta^{(t)}})\|_{L_1(\tpr_\ndata^X)} \|\Theta^* - \Theta^{(t)}\|_2^2 \notag \\
&\leq \risk(\Theta^*) + \frac{K_2}{m^\beta T}\sum_{t=0}^{T-1}\|\nabla_f\risk(f_{\Theta^{(t)}})\|_{L_1(\tpr_\ndata^X)} \max_{t \in \{0,\ldots,T-1\}} \|\Theta^* - \Theta^{(t)}\|_2^2 \notag \\ 
&\leq \risk(\Theta^*) + \frac{2K_2}{m^\beta \sqrt{T}} \sqrt{ \sum_{t=0}^{T-1}\|\nabla_f\risk(f_{\Theta^{(t)}})\|_{L_1(\tpr_\ndata^X)}^2}  \max_{t \in \{0,\ldots,T-1\}} \left(\|\Theta^* - \Theta^{(0)}\|_2^2+\|\Theta^{(0)} - \Theta^{(t)}\|_2^2 \right) \notag \\
&\leq \risk(\Theta^*) + \frac{C_2}{\rho\sqrt{\eta T m}}( \alpha^2 m + \eta T) \notag \\
&\leq \exp\left( -\frac{\alpha \rho m^{1-\beta}}{4} \right) + 
\frac{C_2}{ \rho } \left( \alpha^2 \sqrt{\frac{m}{\eta T}} + \sqrt{\frac{\eta T}{m}} \right). \label{eq:pre_average_loss_bound} 
\end{align}

Finally, we bound the average of $\nabla_\Theta \risk (\Theta^{(t)})^\top (\Theta^{(t)} - \Theta^*)$.
Because $- 2 a^\top b =\|a\|_2^2 + \|b\|_2^2 - \|a+b\|_2^2$ for real vectors $a, b$, we get by setting 
$a = - \eta \nabla_\Theta \risk (\Theta^{(t)})$ and $b = \Theta^{(t)}-\Theta^*$, 
\begin{align*}
\frac{1}{T}\sum_{t=0}^{T-1} \nabla_\Theta \risk (\Theta^{(t)})^\top (\Theta^{(t)} - \Theta^*)
&= \frac{1}{2\eta T}\sum_{t=0}^{T-1}
( \eta^2 \|\nabla_\Theta \risk (\Theta^{(t)}) \|_2^2 
+ \| \Theta^{(t)}-\Theta^* \|_2^2 
- \| \Theta^{(t+1)} - \Theta^*\|_2^2 ) \\
&= \frac{\eta}{2 T}\sum_{t=0}^{T-1} \|\nabla_\Theta \risk (\Theta^{(t)}) \|_2^2
+\frac{1}{2\eta T}\| \Theta^{(0)} - \Theta^* \|_2^2 \\
&\leq \frac{\log(2)}{T} + \frac{\alpha^2 m}{2\eta T}.
\end{align*}

Thus, we get that $\exists C>0$,
\[ \frac{1}{T} \sum_{t=0}^{T-1} \risk(\Theta^{(t)}) 
\leq C\left( \frac{1}{T} + \frac{\alpha^2 m}{\eta T} + 
\exp\left( -\frac{\alpha \rho m^{1-\beta}}{4} \right) + 
 \frac{\alpha^2}{\rho}\sqrt{\frac{m}{\eta T}} + \frac{1}{\rho}\sqrt{\frac{\eta T}{ m}} \right). \]
\end{proof}

We next prove Proposition \ref{proposition:sharper_bound_on_distance} that gives a sharper bound on $\|\Theta^{(T)}- \Theta^{(0)}\|$.
\begin{proof}[Proof of Proposition \ref{proposition:sharper_bound_on_distance}]
  Let $L\in \posintegers$ be a positive integer such that for $T=O(\rho^{-2}\epsilon^{-1}\log^2(1/\epsilon))$, $2^L \leq T < 2^{L+1}$.
  Clearly, $L \leq \log_2T$.
  From Corollary \ref{corollary:faster_convergence_rate}, we get for $l \in \{1,\ldots,L\}$
  \[ \frac{1}{2^{l-1}} \sum_{t=2^{l-1}}^{2^l - 1}\risk(\Theta^{(t)})
    \leq \frac{2}{2^l} \sum_{t=0}^{2^l - 1}\risk(\Theta^{(t)})
    \leq 2C\left( \epsilon + 2^{-l} \rho^{-2}\log^2(1/\epsilon)\right). \]
  Therefore, there exist $2^{l-1} \leq \exists s_l < 2^l$ for $l \in \{1,\ldots,L\}$ such that
  \[ \risk(\Theta^{(s_l)}) \leq 2C\left( \epsilon + 2^{-l} \rho^{-2}\log^2(1/\epsilon)\right).\]
  From the similar argument to the proof of Proposition \ref{prop:gd_iterations}, we get for $a < b \in \posintegers$,
  $\sum_{t=a}^b \| \nabla_\Theta \risk (\Theta^{(t)})\|_2 \leq \sqrt{2(b-a+1)\eta^{-1} \risk(\Theta^{(t)})}$.

  Thus, it follows that since $s_1 = 1$, $\| \nabla_\Theta \risk(\Theta^{(0)})\|_2 \leq \sqrt{m}K_1$ by (\ref{eq:prop1:grad_bound}), and $2^{l+1}-2^{l-1} + 1 \leq 2^{l+1}$,
  \begin{align*}
    \sum_{t=0}^{T-1} \| \nabla_\Theta \risk(\Theta^{(t)})\|_2
    &\leq
    \| \nabla_\Theta \risk(\Theta^{(0)})\|_2
    + \sum_{l=1}^{L-1}\sum_{t=s_l}^{s_{l+1}} \| \nabla_\Theta \risk(\Theta^{(t)})\|_2
    + \sum_{t=s_L}^{T-1} \| \nabla_\Theta \risk(\Theta^{(t)})\|_2 \\
    &\leq
      \sqrt{m}K_1  + \sum_{l=1}^{L}\sqrt{ 2^{3}C \eta^{-1} ( 2^l\epsilon + \rho^{-2}\log^2(1/\epsilon))}\\
    &\leq
      \sqrt{m}K_1  + \log_2(T) \sqrt{ 2^{3}C \eta^{-1} ( T\epsilon + \rho^{-2}\log^2(1/\epsilon))}.
  \end{align*}
  Hence, by setting specific values of $\eta, T,$ and $m$ in Corollary \ref{corollary:faster_convergence_rate}, we get that $\exists C'>0$, 
  \begin{align*}
    \|\Theta^{(t)}-\Theta^{(0)}\|_2 \leq \eta \sum_{t=0}^{T-1} \| \nabla_\Theta \risk(\Theta^{(t)})\|_2
    &\leq \eta \sqrt{m}K_1  + \log_2(T) \sqrt{ 2^{3}C \eta ( T\epsilon + \rho^{-2}\log^2(1/\epsilon))} \\
    &\leq C' \epsilon^{3/4}\log^2(\rho^{-2}\epsilon^{-1}).
  \end{align*}
\end{proof}

\paragraph{Generalization bound.}
A generalization bound can be derived by utilizing the standard analysis of the Rademacher complexity \citep{kp2002}.
We here introduce a function class to be measured by the Rademacher complexity.
Let $l_{\gamma}(v)$ ($\gamma>0$) be the {\it ramp} loss: 
\begin{eqnarray*}
l_{\gamma}(v) \defeq 
\left\{ \begin{array}{ll}
1 & (v < 0), \\
1-v/\gamma & (0 \leq v \leq \gamma), \\
0 & (v > \gamma ).\\
\end{array} \right.
\end{eqnarray*}
Then, a class of all possible ramp losses over $\featuresp \times \labelsp$ attained by the gradient descent (\ref{eq:gd}) 
up to $T$-iterations is defined as follows: 
$\Omega_{\eta,m,T} \defeq \left\{\Theta \mid \|\Theta - \Theta^{(0)}\|_{2} \leq D_{\eta,T,m} \right\}$,
\[ \mathcal{F}_{\eta,m,T}^\gamma \defeq 
\left\{ l_\gamma(yf_\Theta(x)) : \featuresp\times \labelsp \rightarrow [0,1] \mid 
\Theta \in \Omega_{\eta,m,T} \right\}. \]
Here, $D_{\eta,T,m}$ is set to be $\sqrt{2 \eta T \log(2)}$ when considering a general hyperparameter setting in Theorem \ref{theorem:global_convergence}  and is set to be a sharper bound in Proposition \ref{proposition:sharper_bound_on_distance}: $\Theta(\epsilon^{3/4} \log^2(\rho^{-2}\epsilon^{-1}))$ when considering a specific hyperparameter setting in that proposition.

For a given dataset $S=(x_i,y_i)_{i=1}^\ndata$, the Rademacher complexity is defined by 
$\radcomp(\mathcal{F}_{\eta,m,T}^{\gamma}|_{S})\defeq n^{-1}\expec[\sup_{h\in \mathcal{F}_{\eta,m,T}^\gamma} \sum_{i=1}^\ndata \epsilon_i h(x_i,y_i)]$, 
where the expectation is taken over the Rademacher random variables $(\epsilon_i)_{i=1}^\ndata$ which are i.i.d. 
with probabilities $\prob[\epsilon_i=1]=\prob[\epsilon_i=-1]=0.5$.
The following well-known result \citep{kp2002,mohri2012foundations,shalev2014understanding} provides a bound on the expected classification error based on the empirical margin distribution and the Rademacher complexity.
The empirical margin distribution for $S$ is defined as the ratio of examples satisfying $y_i f_\Theta(x_i) \leq \gamma$ in $S$.

\begin{lemma}[\cite{kp2002,mohri2012foundations,shalev2014understanding}] \label{lemma:base_generalization_bound}
Let $\forall \ndata\in \posintegers$, $\forall \gamma>0$, $\forall \eta>0$, $\forall m\in \posintegers$, $\forall T\in \posintegers$, 
and $\forall \delta \in (0,1)$.
Then, with probability at least $1-\delta$ over the random choice of $S$ of size $\ndata$, every $\Theta \in \Omega_{\eta,m,T}$ satisfies
\begin{equation}
 \hspace{-1.5mm}\mathbb{P}_{(X,Y)\sim \tpr}[Yf_\Theta(X)\leq 0] 
\leq \mathbb{P}_{(X,Y)\sim \tpr_\ndata}[Yf_\Theta(X)\leq \gamma] 
+ 2\radcomp(\mathcal{F}_{\eta,m,T}^{\gamma}|_{S}) +3\sqrt{ (2\ndata)^{-1}\log(2/\delta)}. \label{base_generalization_bound} 
\end{equation}
\end{lemma}
To instantiate this bound, we have to provide upper bounds on the empirical margin distribution and the Rademacher complexity.
We first give a bound on the Rademacher complexity.

\begin{proposition} \label{prop:complexity_bound}
Suppose Assumption {\bf(A1)} and {\bf(A2)} hold.
Let $\forall \gamma>0$, $\forall \eta>0$, $\forall m\in \posintegers$, $\forall T\in \posintegers$, $\forall \delta \in (0,1)$, 
and $\forall S$ be examples of size $n$. 
Then, there exists a uniform constant $C>0$ such that with probability at least $1-\delta$ with respect to the initialization of $\Theta^{(0)}$,  
\begin{equation*}
\radcomp(\mathcal{F}_{\eta,m,T}^{\gamma}|_{S}) 
\leq C\gamma^{-1}m^{\frac{1}{2}-\beta}D_{\eta,m,T} (1+K_1+K_2)\sqrt{\frac{d}{\ndata} 
\log\left( \ndata(1+K_1+K_2)(\log(m/\delta)+D_{\eta,m,T}^2) \right) }. 
\end{equation*} 
Moreover, when $\sigma$ is convex and $\sigma(0) = 0$, we can avoid the dependence with respect to the dimension $d$. With probability at least $1-\delta$ 
over a random initialization of $\Theta^{(0)}$,
\begin{equation*}
\radcomp(\mathcal{F}_{\eta,m,T}^{\gamma}|_{S}) 
\leq \frac{8K_1 m^{\frac{1}{2}-\beta}}{\gamma\sqrt{n}}\bigg(D_{\eta,m,T} + \sqrt{\frac{\log(Am/\delta)}{b} } \bigg). 
\end{equation*}
\end{proposition}
\begin{proof}[Proof of Theorem \ref{theorem:generalization_bound}]
We prove this theorem by instantiating inequality (\ref{base_generalization_bound}).
Let $(\Theta^{(t)})_{t=0}^{T-1}$ be a sequence obtained by the gradient descent (\ref{eq:gd}).
Because $(\Theta^{(t)})_{t=0}^{T-1}$ is contained in $\Omega_{\eta,m,T}$, as indicated in Proposition \ref{prop:gd_iterations}, 
inequality (\ref{base_generalization_bound}) holds for this sequence.
As for the Rademacher complexity in (\ref{base_generalization_bound}), we can utilize Proposition \ref{prop:complexity_bound}. 
Thus, the resulting problem is to prove the convergence of the empirical margin distribution: 
$\mathbb{P}_{(X,Y)\sim \tpr_\ndata}[Yf_{\Theta^{(t)}}(X)\leq \gamma]$.
We here give its upper-bound below.
\[ 0.5 \left| y_i - 2p_{\Theta}(Y=1|x_i) + 1 \right| \geq (1+\exp(\gamma))^{-1} 
\Longleftrightarrow y_i f_\Theta(x_i) \leq \gamma. \]
Therefore, from Markov's inequality,
\begin{align*}
\mathbb{P}_{(X,Y)\sim \tpr_\ndata}[Yf_{\Theta^{(t)}}(X)\leq \gamma] 
&= \mathbb{P}_{(X',Y')\sim \tpr_\ndata}\left[ \frac{1}{2}\left| Y' - 2p_\Theta(Y=1|X') + 1 \right| \geq \frac{1}{1+\exp(\gamma)}  \right] \\
&\leq (1+\exp(\gamma)) \| \nabla_f \risk(f_\Theta) \|_{L_1(\tpr_\ndata^X)}.
\end{align*}
Combining this inequality with Lemma \ref{lemma:base_generalization_bound}, then for $\forall t \in \{0,\ldots,T-1\}$,
\begin{equation*}
 \mathbb{P}_{(X,Y)\sim \tpr}[Yf_{\Theta^{(t)}}(X)\leq 0] 
\leq (1+\exp(\gamma)) \| \nabla_f \risk(f_{\Theta^{(t)}}) \|_{L_1(\tpr_\ndata^X)}
+ 2\radcomp(\mathcal{F}_{\eta,m,T}^{\gamma}|_{S}) + 3\sqrt{\frac{\log(2/\delta)}{2\ndata}}. 
\end{equation*}
Noting that $\eta, m,$ and $T$ satisfy the conditions in Theorem \ref{theorem:global_convergence}, 
we can complete the proof by taking the average over $t \in \{0,\ldots,T-1\}$ and 
applying Proposition \ref{prop:complexity_bound} and Theorem \ref{theorem:global_convergence}.
\end{proof}

\section{Proofs for Global Convergence}\label{sec:proofs_global_convergence}
\subsection{Proof of Proposition \ref{prop:smoothness}}
\begin{proof}[Proof of Proposition \ref{prop:smoothness}]
We first show the smoothness of $f_{\Theta}(x)$ with respect to $\Theta$ for $\forall x \in \featuresp$, $(\|x\|_2 \leq 1)$.
Noting that $\nabla_\Theta^2 f_\Theta(x) = \mathrm{diag}\left( \frac{1}{m^\beta} a_r \sigma''(\theta_r ^\top x)x x^{\top}\right)_{r=1}^m$, 
we get for $\tau = (\tau_r)_{r=1}^m$ such that $\sum_{r=1}^m \|\tau_r\|_2^2 = 1$ $(\tau_r \in \realsp^\fdim)$, 
\begin{align*}
\left| \tau^\top \nabla_\Theta^2 f_{\Theta}(x) \tau \right|
&= \left| \sum_{r=1}^m \tau_r^\top \partial_{\theta_r}^2 f_\Theta(x) \tau_r \right| \\
&\leq \frac{1}{m^\beta}\sum_{r=1}^m \left|\sigma''(\theta_r ^\top x)\right| \left|\tau_r^\top x \right|^2 \\
&\leq \frac{K_2}{m^\beta} \sum_{r=1}^m \|\tau_r\|_2^2\\
&=\frac{K_2}{m^\beta}.
\end{align*}
This means that for $\tau = (\tau_r)_{r=1}^m$, $(\tau_r \in \realsp^\fdim)$, 
\begin{equation}
    \left| f_{\Theta+\tau}(x) - (f_{\Theta}(x) + \nabla_{\Theta}f_{\Theta}(x)^\top \tau )\right| 
    \leq \frac{K_2}{m^\beta} \| \tau \|_2^2 = \frac{K_2}{m^\beta} \sum_{r=1}^m \|\tau_r\|_2^2. \label{eq:prop1:f_smoothness}
\end{equation}
Let us define $g_x(\tau)$ as the second-order term of Taylor's expansion of $f_{\Theta}(x)$ with respect to $\Theta$:
\[ f_{\Theta + \tau}(x) = f_{\Theta}(x) + \nabla_{\Theta} f_{\Theta}(x)^{\top} \tau + g_x(\tau). \]
From the inequality (\ref{eq:prop1:f_smoothness}), we have $| g_x(\tau) | \leq \frac{K_2 \|\tau\|_2^2}{m^\beta}$.
Then, by the smoothness of $l(\zeta,y)$ with respect to $\zeta$ and $|\partial_\zeta^2 l(\zeta,y)| \leq 1/4$, we get
\begin{align}
\Bigl| l(f_{\Theta+\tau}(x),y) 
- ( l(f_{\Theta}(x),y) \Bigr. 
&\Bigl.  + \partial_\zeta l(f_\Theta(x),y) (\nabla_\Theta f_{\Theta}(x)^\top \tau + g_x(\tau)) ) \Bigr|   \notag \\
& \leq \frac{1}{4}\left| \nabla_\Theta f_{\Theta}(x)^\top \tau + g_x(\tau) \right|^2 \notag \\
&\leq \frac{1}{2} \left( \| \nabla_\Theta f_{\Theta}(x)\|_2^2 + \frac{K_2^2\|\tau\|_2^2}{m^{2\beta}} \right)\|\tau\|_2^2. \notag
\end{align}
By the triangle inequality, we get
\begin{align}
| l(f_{\Theta+\tau}(x),y) 
&- \left( l(f_{\Theta}(x),y) 
+ \partial_\zeta l(f_\Theta(x),y) \nabla_\Theta f_{\Theta}(x)^\top \tau \right) | \notag \\
&\leq | \partial_\zeta l(f_\Theta(x),y)g_x(\tau)| 
+ \frac{1}{2} \left( \| \nabla_\Theta f_{\Theta}(x)\|_2^2 + \frac{K_2^2\|\tau\|_2^2}{m^{2\beta}} \right)\|\tau\|_2^2 \notag \\
&\leq \frac{1}{2} \left( \| \nabla_\Theta f_{\Theta}(x)\|_2^2 
+ \frac{2K_2}{m^\beta}
+ \frac{K_2^2\|\tau\|_2^2}{m^{2\beta}} \right)\|\tau\|_2^2 \notag \\
&\leq \frac{1}{2} \left( \frac{K_1^2}{m^{2\beta-1}}
+ \frac{2K_2}{m^\beta}
+ \frac{K_2^2\|\tau\|_2^2}{m^{2\beta}} \right)\|\tau\|_2^2, \label{eq:prop1:l_smoothness}
\end{align}
where for the second inequality, we used $|\partial_\zeta l(\zeta,y)|\leq 1$ 
and for the last inequality, we used 
\begin{align*}
\| \nabla_\Theta f_\Theta(x)\|_2^2 
= \sum_{r=1}^m \left\| \frac{1}{m^\beta}\sigma'(\theta_r^\top x) x \right \|_2^2
\leq \sum_{r=1}^m \frac{1}{m^{2\beta}}| \sigma'(\theta_r^\top x)  |_2^2 
\leq \frac{K_1^2}{m^{2\beta-1}}.
\end{align*}

We here set $\tau = -\eta \nabla_\Theta \risk(\Theta)$.
The right hand side of (\ref{eq:prop1:l_smoothness}) is upper bounded by 
\[ \frac{1}{2} \left( \frac{K_1^2}{m^{2\beta-1}}
+ \frac{2K_2}{m^\beta}
+ \frac{\eta^2 K_1^2K_2^2}{m^{4\beta-1}} \right)\|\tau\|_2^2. \]
because
\begin{align}
\| \nabla_\Theta \risk(\Theta) \|_2^2
&= \sum_{r=1}^m \| \partial_{\theta_r} \risk(\Theta) \|_2^2 \notag\\
&= \sum_{r=1}^m \left\| \frac{1}{\ndata}\sum_{i=1}^\ndata \partial_\zeta l (f_\Theta(x_i),y_i) \partial_{\theta_r}f_{\Theta}(x_i) \right\|_2^2 \notag\\
&\leq \sum_{r=1}^m \left( \frac{1}{\ndata}\sum_{i=1}^\ndata |\partial_\zeta l (f_\Theta(x_i),y_i)|\|\partial_{\theta_r}f_{\Theta}(x_i)\|_2 \right)^2 \notag\\
&\leq \left( \frac{1}{\ndata}\sum_{i=1}^\ndata |\partial_\zeta l (f_\Theta(x_i),y_i)| \right)^2 \sum_{r=1}^m \left(\max_{j\in\{1,\ldots,\ndata\}}\|\partial_{\theta_r}f_{\Theta}(x_j)\|_2 \right)^2 \notag\\
&= \| \nabla_f \risk(f_\Theta) \|_{L_{1}(\tpr_\ndata^X)}^2 \sum_{r=1}^m \left(\max_{j\in\{1,\ldots,\ndata\}} \frac{1}{m^\beta} |\sigma'(\theta_r^\top x_j)| \|x_j\|_2 \right)^2 \notag\\
& \leq \| \nabla_f \risk(f_\Theta) \|_{L_{1}(\tpr_\ndata^X)}^2 m^{1-2\beta} K_1^2 \notag\\
& \leq m^{1-2\beta} K_1^2. \label{eq:prop1:grad_bound}
\end{align}

Therefore, we get
\begin{align}
| l(f_{\Theta+\tau}(x),y) 
&- \left( l(f_{\Theta}(x),y) 
-\eta \partial_\zeta l(f_\Theta(x),y) \nabla_\Theta f_{\Theta}(x)^\top \nabla_\Theta \risk(\Theta) \right) | \notag \\
&\leq \frac{1}{2} \left( \frac{K_1^2}{m^{2\beta-1}}
+ \frac{2K_2}{m^\beta}
+ \frac{\eta^2 K_1^2K_2^2}{m^{4\beta-1}} \right)\eta^2 \| \nabla_\Theta \risk(\Theta)\|_2^2 \notag\\
&\leq \frac{1}{2m^{2\beta-1}} \left( K_1^2 + 2K_2 + K_1^2K_2^2 \right)\eta^2 \| \nabla_\Theta \risk(\Theta)\|_2^2, \label{eq:prop:l_smoothness2}
\end{align}
where we used $\beta \in [0,1)$ and $\eta \leq m^\beta$ for the last inequality.

Noting that from the definition of kernel smoothing of functional gradients (\ref{eq:kernel_smoothing}), 
we see
\begin{align*} 
\nabla_\Theta f_{\Theta}(x)^\top \nabla_\Theta \risk(\Theta) 
&= \nabla_\Theta f_{\Theta}(x)^\top \left(\frac{1}{\ndata}\sum_{i=1}^\ndata \partial_{\zeta}l(f_\Theta(x_i),y_i) \nabla_\Theta f_{\Theta}(x_i) \right)  \\
&= T_{k_{\Theta}} \nabla_f \risk(f_\Theta)(x).
\end{align*}
Therefore, by taking the expectation of (\ref{eq:prop:l_smoothness2}) according to the empirical distribution $\tpr_{\ndata}$, we get 
\begin{align*}
\Bigl| \risk(f_{\Theta+\tau}) \Bigr. &- \Bigl. \left(\risk(f_{\Theta}) - \eta \pd< \nabla_f \risk(f_{\Theta}), T_{k_{\Theta}}\nabla_f \risk(f_\Theta)>_{L_2(\tpr_\ndata^X)} \right) \Bigr| \\
&\leq  \frac{\eta^2}{2m^{2\beta-1}} \left( K_1^2 + 2K_2 + K_1^2K_2^2 \right) \| \nabla_\Theta \risk(\Theta)\|_2^2.
\end{align*}
This completes the proof of the statement (i).

From the convexity of $l(\zeta,y)$ with respect to $\zeta$, we have
\begin{align*}
l(f_{\Theta^*}(X), Y ) 
&= l\left(f_{\Theta}(X)  + \nabla_\Theta f_{\Theta}(X)^\top (\Theta^* - \Theta) + g_X(\Theta^* - \Theta), Y \right) \\
&\geq l(f_{\Theta}(X),Y)  + \partial_\zeta l( f_\Theta(X),Y) \left(\nabla_\Theta f_{\Theta}(X)^\top (\Theta^* - \Theta) + g_X(\Theta^* - \Theta) \right) \\
&\geq l(f_{\Theta}(X),Y)  + \nabla_\Theta l( f_\Theta(X),Y)^\top (\Theta^* - \Theta) 
- \left| \partial_\zeta l( f_\Theta(X),Y) \right| \frac{K_2 \| \Theta^* - \Theta \|_2^2}{ m^\beta },
\end{align*}
where we used $| g_X(\Theta^* - \Theta) | \leq \frac{K_2 \| \Theta^* - \Theta \|_2^2}{ m^\beta }$.
Thus, by taking the expectation with respect to $(X,Y)\sim \tpr_\ndata$, we get
\[ \risk(\Theta^*) \geq 
\risk(\Theta) + \nabla_\Theta \risk(\Theta)^\top (\Theta^* - \Theta) 
- \frac{K_2}{m^\beta}\|\nabla_f\risk(f_{\Theta})\|_{L_1(\tpr_\ndata^X)} \|\Theta^* - \Theta\|_2^2. \] 
This finishes the proof of the statement (ii).
\end{proof}

\subsection{Proof of Proposition \ref{prop:kernel_positivity}}
\begin{proof}[Proof of Proposition \ref{prop:kernel_positivity}]
Set $Z_{r,i} \defeq y_i \partial_\theta \sigma(\theta_r^{(0)}x_i)^\top v(\theta_r^{(0)})$.
We find clearly $|Z_{r,i}| \leq K_1$ from Assumption \ref{assumption:convergence_analysis}.
By applying Hoeffding's inequality to $Z_{r,i}$ for each $i\in\{1,\ldots,\ndata\}$ 
and taking an union bound, we have
\[ \prob_{\Theta^{(0)}}\left[ \max_{i\in\{1,\ldots,\ndata\}} \left|\frac{2}{m}\sum_{r=1}^{m/2} Z_{r,i} 
- \expec_{\theta^{(0)}_r}[Z_{r,i}] \right| > \frac{\rho}{2} \right] 
\leq 2\ndata \exp\left( - \frac{\rho^2 m}{16K_1^2} \right). \]
In other words, since $m \geq \frac{16 K_1^2}{\rho^2}\log \frac{2\ndata}{\delta}$, we have with probability $1-\delta$, 
\[ \max_{i\in\{1,\ldots,\ndata\}} \left|\frac{2}{m}\sum_{r=1}^{m/2} Z_{r,i} - \expec_{\theta^{(0)}_r}[Z_{r,i}] \right| 
\leq \frac{\rho}{2}. \]
Therefore, using Assumption \ref{assumption:convergence_analysis} {\bf(A4)} and noting $\Theta^{(0)}=(\theta_r)_{r=1}^m$ is symmetrically initialized, 
we get with probability $1-\delta$ for $\forall i \in \{1,\ldots,\ndata\}$,
\begin{equation} 
\frac{1}{m}\sum_{r=1}^{m} y_i \partial_\theta \sigma(\theta_r^{(0)}x_i)^\top v(\theta_r^{(0)}) \geq \frac{\rho}{2}. \label{eq:high_prob_separability}
\end{equation}

In the following proof, we assume $\Theta^{(0)}=(\theta_r^{(0)})_{r=1}^m$ satisfies this inequality.
We get from the $K_2$-Lipschitz continuity of $\sigma'$ 
that for $\Theta=(\theta_r)_{r=1}^m$ satisfying $\| \Theta - \Theta^{(0)}\|_{2,1} \leq \frac{m\rho}{4K_2}$,
\begin{align*}
&\hspace{-10mm}\left| \frac{1}{m} \sum_{r=1}^m y_i \sigma'(\theta_r^\top x_i)x_i^\top v(\theta_r^{(0)})
- \frac{1}{m} \sum_{r=1}^m y_i \sigma'(\theta_r^{(0)\top} x_i)x_i^\top v(\theta_r^{(0)}) \right|  \\
&\leq \frac{1}{m}\sum_{r=1}^m \left| y_ix_i^\top v_r(\theta_r^{(0)})( \sigma'(\theta_r^\top x_i) - \sigma'(\theta_r^{(0)\top}x_i))\right| \\
&\leq \frac{1}{m}\sum_{r=1}^m K_2| (\theta_r - \theta_r^{(0)})^\top x_i | \\
&\leq \frac{K_2}{m} \| \Theta - \Theta^{(0)} \|_{2,1} \\
&\leq \frac{\rho}{4}.
\end{align*}
This means that there exists $(v_r)_{r=1}^m$ such that $\| v_r \|_2 \leq 1$ $(\forall r \in \{1,\ldots,m\})$ and
for $\forall \Theta=(\theta_r)_{r=1}^m$ satisfying $\| \Theta - \Theta^{(0)}\|_{2,1} \leq \frac{m\rho}{4K_2}$,
\[ \frac{1}{m} \sum_{r=1}^m y_i \partial_{\theta}\sigma(\theta_r^\top x_i)^\top v_r \geq \frac{\rho}{4}, \hspace{5mm} \forall i \in \{1,\ldots,\ndata\}. \] 

Then, we get the following bound: for $\forall (\alpha_i)_{i=1}^\ndata$ $(\alpha_i \in (0,1))$, 
\begin{equation}
\frac{1}{m} \sum_{i=1}^\ndata \sum_{r=1}^m y_i \alpha_i \partial_\theta \sigma(\theta_r^\top x_i)^\top v_r 
= \frac{1}{m} \sum_{i=1}^\ndata \alpha_i \sum_{r=1}^m y_i \partial_\theta \sigma(\theta_r^\top x_i)^\top v_r 
\geq \frac{\rho}{4} \sum_{i=1}^\ndata \alpha_i
> 0. \label{eq:prop2:usefule_ineq}
\end{equation}

Noting that $\nabla_f \risk(f_\Theta)(x_i) = \frac{-y_i}{1+\exp(y_i f_\Theta(x_i))}$, we get
\begin{align*}
\pd< \nabla_f \risk(f_{\Theta}), T_{k_{\Theta}}\nabla_f \risk(f_\Theta)>_{L_2(\tpr_\ndata^X)} 
&= \frac{1}{\ndata^2} \sum_{i,j=1}^\ndata k_{\Theta}(x_i,x_j)\nabla_f \risk(f_\Theta)(x_i) \nabla_f \risk(f_\Theta)(x_j) \\
&= \frac{1}{\ndata^2} \left\| \sum_{i=1}^\ndata  \nabla_f \risk(f_\Theta)(x_i) \nabla_\Theta f_\Theta(x_i) \right\|_2^2 \\
&= \frac{1}{\ndata^2} \sum_{r=1}^m \left\| \sum_{i=1}^\ndata  \nabla_f \risk(f_\Theta)(x_i) \partial_{\theta_r} f_\Theta(x_i) \right\|_2^2 \\
&= \frac{1}{\ndata^2} \sum_{r=1}^m \left\| \frac{1}{m^{\beta}} \sum_{i=1}^\ndata  \nabla_f \risk(f_\Theta)(x_i) \partial_\theta \sigma(\theta_r^\top x_i) \right\|_2^2 \\
&\geq \frac{1}{\ndata^2} \sum_{r=1}^m \left( \frac{1}{m^{\beta}} \sum_{i=1}^\ndata \nabla_f \risk(f_\Theta)(x_i) \partial_\theta \sigma(\theta_r^\top x_i)^\top v_r \right)^2 \\
&\geq \frac{m}{\ndata^2} \left( \frac{1}{m^{1+\beta}} \sum_{i=1}^\ndata \sum_{r=1}^m \nabla_f \risk(f_\Theta)(x_i) \partial_\theta \sigma(\theta_r^\top x_i)^\top v_r \right)^2 \\
&\geq \frac{m^{1-2\beta}\rho^2}{16\ndata^2} \left( \sum_{i=1}^\ndata \frac{1}{1+\exp(y_if_\Theta(x_i))}\right)^2,
\end{align*}
where we used $\|v_r\|_2 \leq 1$ for the first inequality, the convexity of $\|\cdot\|_2^2$ for the second inequality, and (\ref{eq:prop2:usefule_ineq}) for the last inequality.
We can find that this inequality finishes the proof because 
\[ \frac{1}{1+\exp(f_\Theta(x_i)y_i)} = \frac{1}{2}\left| y_i - 2p_{\Theta}(Y=1 \mid x_i)+ 1\right|. \]
\end{proof}

\subsection{Proof of Proposition \ref{prop:gd_iterations}}
The proof of Proposition \ref{prop:gd_iterations} is based on the traditional convergence analysis of gradient descent for smooth objective functions in finite-dimensional space.
\begin{proof}[Proof of Proposition \ref{prop:gd_iterations}]
We first specify the smoothness of the logistic loss function.
We set $\phi(v)= \log( 1 + \exp(-v))$ and $l(y,f_\Theta(x)) = \phi(yf_\Theta(x))$.
By the simple calculation, we get that for $r,s \in \{1,\ldots,m\}$,
\begin{align*}
\frac{\partial^2}{\partial\theta_r \partial \theta_s} l(y,f_\Theta(x))
&= \phi''(yf_\Theta(x))\frac{a_r a_s}{m^{2\beta}} \sigma'(\theta_r^\top x) \sigma'(\theta_s^\top x)x x^\top \\
&+ \mathbf{1}[r=s]\frac{y}{m^\beta}\phi'(yf_\Theta(x)) a_r \sigma''(\theta_r^\top x)x x^\top.
\end{align*}  
Noting that $\|\phi'\|_{\infty}\leq 1$ and $\|\phi''\|_{\infty}\leq \frac{1}{4}$, 
we can see that the maximum eigen-value of 
$( \partial^2 l(y,f_\Theta(x))  / \partial\theta_r \partial \theta_s )_{r,s=1}^m$ is upper bounded by 
\[ M \defeq \frac{1}{4m^{2\beta-1}}(K_1^2 + K_2). \]
Indeed, for $v=(v_r)_{r=1}^m$ such that $\sum_{r=1}^m\|v_r\|_2^2 \leq 1$, ($v_r \in \realsp^d$), we have
\begin{align*}
\sum_{r,s=1}^m v_r^\top \frac{\partial^2 l(y,f_\Theta(x))}{\partial\theta_r \partial \theta_s} v_s 
&=\frac{\phi''(yf_\Theta(x))}{m^{2\beta}}\left( \sum_{r=1}^m a_r \sigma'(\theta_r^\top x)v_r^\top x \right)^2 
+ \frac{y}{m^{\beta}}\phi'(yf_\Theta(x)) \sum_{r=1}^m a_r \sigma''(\theta_r^\top x)(v_r^\top x)^2 \\
&\leq \frac{K_1^2}{4m^{2\beta}}\left(\sum_{r=1}^m \|v_r\|_2\right)^2 
+ \frac{K_2}{m^\beta} \sum_{r=1}^m \|v_r\|_2^2\\
&\leq \frac{K_1^2}{4m^{2\beta}}\left( \sqrt{m}\|v\|_2 \right)^2 
+ \frac{K_2}{m^\beta} \\
&\leq \frac{K_1^2}{4m^{2\beta-1}} + \frac{K_2}{m^\beta} \\
&\leq \frac{1}{4m^{2\beta-1}}(K_1^2 + K_2).
\end{align*}
Therefore, the loss function $\risk(\Theta)$ is $M$-Lipschitz smooth with respect to $\Theta$, that is, for 
\[ \risk(\Theta') \leq \risk(\Theta) + \pd< \nabla \risk(\Theta), \Theta'-\Theta>_2 + \frac{M}{2}\|\Theta'-\Theta\|_2^2. \] 
Plugging $\Theta=\Theta^{(t)}$ and $\Theta'=\Theta^{(t+1)}= \Theta- \eta \nabla_\Theta \risk(\Theta^{(t)})$ into this inequality, we get
\begin{align*}
\risk(\Theta^{(t+1)}) 
&\leq \risk(\Theta^{(t)}) - \eta\left(1 - \frac{\eta M}{2}\right) \| \nabla \risk(\Theta^{(t)}) \|_2^2 \\
&\leq \risk(\Theta^{(t)}) - \frac{\eta}{2} \| \nabla \risk(\Theta^{(t)}) \|_2^2,
\end{align*}
where we used $\eta \leq 1/M$ for the last inequality.
By summing this inequality over $t \in \{0,\ldots,T-1\}$ and multiplying by $\frac{2}{\eta T}$, we get
\begin{equation}
\frac{1}{T}\sum_{t=0}^{T-1}\| \nabla \risk(\Theta^{(t)}) \|_2^2  
\leq \frac{2}{\eta T}\risk(\Theta^{(0)}) 
=  \frac{2}{\eta T}\log(2), \label{eq:prop3:conv_rate}
\end{equation}  
where we used $\risk(\Theta^{(0)}) = \log(2)$.
Therefore, we have that from equation (\ref{eq:prop3:conv_rate}),
\begin{align*}
\| \Theta^{(t)} - \Theta^{(0)}\|_{2}
&\leq \eta \sum_{t=0}^{T-1} \left \| \nabla_{\Theta} \risk(\Theta^{(t)}) \right\|_2\\
&\leq \eta \sqrt{T} \sqrt{ \sum_{t=0}^{T-1} \left\| \nabla_{\Theta} \risk(\Theta^{(t)}) \right\|_2^2 } \\
&\leq \sqrt{2 \eta T \log(2)}.
\end{align*}

The last statement of Proposition \ref{prop:gd_iterations} immediately follows from this and the following inequality.
\begin{align*}
\| \Theta^{(t)} - \Theta^{(0)}\|_{2,1}
&= \sum_{r=1}^m \| \theta_r^{(t)} - \theta_r^{(0)} \|_2 \\
&\leq \sqrt{m}\sqrt{\sum_{r=1}^m \| \theta_r^{(t)} - \theta_r^{(0)} \|_2^2} \\
&= \sqrt{m}\| \Theta^{(t)} - \Theta^{(0)}\|_{2}.
\end{align*}
\end{proof}

\section{Proofs for Generalization Bounds}\label{sec:proofs_generalization_bounds}
\begin{proof}[Proof of Proposition \ref{prop:complexity_bound}]
In this proof, we denote $\mathcal{F}=\mathcal{F}_{\eta,m,T}^\gamma$ and $\Omega = \Omega_{\eta,m,T}$ for simplicity.
We define $\mathcal{F}_y \defeq \{ h(\cdot,y): \featuresp \rightarrow [0,1] \mid h \in \mathcal{F} \}$.
Then, for a given dataset $S=(x_i,y_i)_{i=1}^\ndata$, we notice that 
$\radcomp(\mathcal{F}|_{S}) \leq \radcomp(\mathcal{F}_{1}|_{X}) + \radcomp(\mathcal{F}_{-1}|_{X})$, where $X=(x_i)_{i=1}^\ndata$.
Thus, it is enough to provide an upper bound on $\radcomp(\mathcal{F}_{1}|_{X})$ because a bound on the other complexity can be 
also derived in the same way.

We first give a uniform high probability bound on the initialization $\|\theta_r^{(0)}\|_2$ for $\forall r \in \{1,\ldots,m\}$.
We get from {\bf(A2)}, for $t>0$,
\begin{align*}
\prob\left[ \max_{r\in \{1,\ldots,m\}} \|\theta_r^{(0)}\|_2 \geq t \right] 
\leq \sum_{r=1}^m \prob\left[ \|\theta_r^{(0)}\|_2 \geq t \right] 
\leq m A \exp(-bt^2).
\end{align*}
Thus, by choosing $t$ so that $\delta = m A \exp(-bt^2)$, we confirm that with probability at least $1-\delta$, 
\[ \max_{r\in \{1,\ldots,m\}} \|\theta_r^{(0)}\|_2 \leq R \defeq \sqrt{\frac{1}{b}\log\left( \frac{mA}{\delta}\right)}. \]

We introduce several notations.
Fix $R_0 > 0$.
We denote $\overline{\theta}=(\theta, \theta') \in \realsp^{2d}$, ($\theta, \theta' \in \realsp^{d}$) and, define for $\overline{\theta}$, 
\[ g_{\overline{\theta}}(x) \defeq \frac{\sigma(\theta^\top x) - \sigma(\theta^{'\top} x)}{\|\theta - \theta'\|_2}. \]
When $\theta=\theta'$, we define $g_{\overline{\theta}}(x)=0$.
From the Lipschitz continuity of $\sigma$, the range of $g_{\overline{\theta}}$ is $[-K_1,K_1]$.
Moreover, we define 
\begin{align*}
\overline{\Omega} &\defeq \{ \overline{\theta} \in \realsp^{2d} \mid\ \|\theta\|_2, \|\theta'\|_2 \leq R + D_{\eta,m,T},\ 
\| \theta - \theta' \|_2  \leq D_{\eta,m,T} \}, \\
\overline{\Omega}_{+} &\defeq \{ \overline{\theta} \in \overline{\Omega} \mid R_0 < \|\theta- \theta'\|_2 \leq D_{\eta,m,T} \}, \\
\overline{\Omega}_{-} &\defeq \{ \overline{\theta} \in \overline{\Omega} \mid\ \|\theta- \theta'\|_2 \leq R_0 \}, \\
\mathcal{G}_{+} &\defeq \left\{ g_{\overline{\theta}}: 
\featuresp \rightarrow [-K_1, K_1] \mid\ \overline{\theta} \in \overline{\Omega}_{+} \right\}, \\
\mathcal{G}_{-} &\defeq \left\{ g_{\overline{\theta}}: 
\featuresp \rightarrow [-K_1, K_1] \mid\ \overline{\theta} \in \overline{\Omega}_{-} \right\}, \\ 
\mathcal{H} &\defeq \{ f_{\Theta} : \featuresp \rightarrow \realsp\mid\ \Theta \in \Omega \}. 
\end{align*}
Clearly, we see
\[ \overline{\Omega} = \overline{\Omega}_{-} \cup  \overline{\Omega}_{+}
\ \textrm{and}\  \left\{ g_{\overline{\theta}} \mid\ \overline{\theta} \in \overline{\Omega} \right\} 
= \mathcal{G}_{-} \cup \mathcal{G}_{+}. \]
From the Lipschitz continuity of $l_\gamma$, we find $\radcomp(\mathcal{F}_1 |_X) \leq \gamma^{-1}\radcomp(\mathcal{H}|_X)$.

We now derive an upper bound on the Rademacher complexity.
Set $C_M \defeq m^{\frac{1}{2}-\beta}D_{\eta,m,T}$.
\begin{align}
\radcomp(\mathcal{H}|_{X}) 
&= \frac{1}{\ndata} \expec\left[ \sup_{\Theta \in \Omega} \sum_{i=1}^\ndata \epsilon_i f_\Theta(x_i)\right] \notag\\
&= \frac{1}{\ndata} \expec\left[ \sup_{\Theta \in \Omega} \sum_{i=1}^\ndata \epsilon_i (f_\Theta(x_i) - f_{\Theta^{(0)}}(x_i))\right] \notag\\
&= \frac{1}{\ndata} \expec\left[ \sup_{\Theta \in \Omega} \sum_{i=1}^\ndata \epsilon_i \frac{1}{m^\beta} \sum_{r=1}^m a_r \left(\sigma(\theta_r^\top x_i)- \sigma(\theta_r^{(0)\top}x_i) \right) \right] \notag\\
&= \frac{C_M}{\ndata} \expec\left[ \sup_{\Theta \in \Omega} \sum_{i=1}^\ndata \epsilon_i \sum_{r=1}^m \frac{\|\theta_r - \theta_r^{(0)}\|_2}{C_M m^\beta}  a_r
   \frac{\sigma(\theta_r^\top x_i)- \sigma(\theta_r^{(0)\top}x_i)}{\|\theta_r - \theta_r^{(0)}\|_2} \right], \label{eq:rademacher_H_bound_1}  
\end{align}
where we used the fact that $f_{\Theta^{(0)}}(x_i)$ is a constant in the expectation for the second equality.

Since, for $\Theta \in \Omega$, 
\[ \sum_{r=1}^m \frac{\|\theta_r - \theta_r^{(0)}\|_2}{C_M m^\beta} \leq \frac{ m^{\frac{1}{2}-\beta}\|\Theta-\Theta^{(0)}\|_2}{C_M} \leq 1, \]
equation (\ref{eq:rademacher_H_bound_1}) can be upper-bounded by the Rademacher complexity of the convex hull.
Hence,
\begin{align}
\radcomp(\mathcal{H}|_{X}) 
 &\leq \frac{C_M}{\ndata} \expec\left[ \sup_{\substack{\Theta \in \Omega \\ \sum_{r=1}^m b_r \leq 1, b_r \in [0,1]}}
  \sum_{i=1}^\ndata \epsilon_i \sum_{r=1}^m b_r a_r
                              \frac{\sigma(\theta_r^\top x_i)- \sigma(\theta_r^{(0)\top}x_i)}{\|\theta_r - \theta_r^{(0)}\|_2} \right] \notag \\
&\leq \frac{C_M}{\ndata} \expec\left[ \sup_{\substack{ (\theta_r,\theta_r')_{r=1}^m \in \overline{\Omega}^m \\ \sum_{r=1}^m b_r \leq 1, b_r \in [0,1]}}
  \sum_{i=1}^\ndata \epsilon_i \sum_{r=1}^m b_r
  \frac{\sigma(\theta_r^\top x_i)- \sigma(\theta_r^{'\top}x_i)}{\|\theta_r - \theta_r^{'}\|_2} \right] \notag \\
&= \frac{C_M}{\ndata} \expec\left[ \sup_{ (\theta,\theta') \in \overline{\Omega} } 
\sum_{i=1}^\ndata \epsilon_i \frac{\sigma(\theta^\top x_i) - \sigma(\theta^{'\top} x_i)}{\|\theta - \theta'\|_2} \right] \notag\\
&= \frac{C_M}{\ndata} \expec\left[ \sup_{ \overline{\theta} \in \overline{\Omega} } 
\sum_{i=1}^\ndata \epsilon_i g_{\overline{\theta}}(x_i) \right] \notag\\
&\leq \frac{C_M}{\ndata} \expec\left[ 
\sup_{ g \in \mathcal{G}_{-} } \sum_{i=1}^\ndata \epsilon_i g(x_i) 
+ \sup_{ g \in \mathcal{G}_{+} } \sum_{i=1}^\ndata \epsilon_i g(x_i)\right] 
= C_M \left(\radcomp(\mathcal{G}_{-}|_X) + \radcomp( \mathcal{G}_{+}|_X) \right). \label{eq:rademacher_H_bound}
\end{align}
We used that for $\Theta \in \Omega$, $(\theta_r,\theta_r^{(0)}) \in \overline{\Omega}$ ($\forall r\in \{1,\ldots,m\}$) because 
$\|\Theta -\Theta^{(0)}\|_2 \leq D_{\eta,m,T}$.
Moreover, the term $a_r$ disappeared by the symmetry.
We used the fact that the convex hull of a hypothesis class does not increase the Rademacher complexity for the first equality.

We next derive an upper bound on the Rademacher complexity $\radcomp(\mathcal{G}_{+}|_X)$ 
through the covering number $\mathcal{N}(\mathcal{G}_{+}|_X,\epsilon,\|\cdot\|_\infty)$ and Lemma \ref{lemma:dudley_integral}.
To this end, we investigate the sensitivity of $\|g_{\overline{\theta}}\|_\infty$ with respect to $\overline{\theta}$ as follows.

Let $\overline{\theta}_1=(\theta_1, \theta_1') \in \overline{\Omega}_+$ and 
$\overline{\theta}_2=(\theta_2,\theta_2') \in \overline{\Omega}_+$ be parameters such that 
\[ \| \overline{\theta}_1 - \overline{\theta}_2 \|_2 = \sqrt{ \| \theta_1 - \theta_2 \|_2^2 + \| \theta_1' - \theta_2' \|_2^2} \leq \epsilon. \]
This leads to 
\[ \|\theta_1 - \theta_2\|_2, \|\theta_1' - \theta_2'\|_2 \leq \epsilon \ \textrm{and}\ 
\left| \|\theta_1 - \theta_1'\|_2 - \|\theta_2 - \theta_2'\|_2 \right| \leq 2\epsilon. \]
We get from these inequalities that for $\|x\|_2 \leq 1$,
\begin{align*}
| g_{\overline{\theta}_1}(x) - g_{\overline{\theta}_2}(x)| 
&= \frac{\left| \|\theta_2 - \theta_2'\|_2 (\sigma(\theta_1^\top x) - \sigma(\theta_1^{'\top} x))
- \|\theta_1 - \theta_1'\|_2 (\sigma(\theta_2^\top x) - \sigma(\theta_2^{'\top} x)) \right| }
{\|\theta_1 - \theta_1'\|_2 \|\theta_2 - \theta_2'\|_2} \\
&\leq \frac{\left| (\|\theta_2 - \theta_2'\|_2 - \|\theta_1 - \theta_1'\|_2 ) (\sigma(\theta_1^\top x) - \sigma(\theta_1^{'\top} x)) \right|}
{\|\theta_1 - \theta_1'\|_2 \|\theta_2 - \theta_2'\|_2} \\
&+ \frac{\left| \|\theta_1 - \theta_1'\|_2 (\sigma(\theta_1^\top x) - \sigma(\theta_1^{'\top} x) 
- \sigma(\theta_2^\top x) + \sigma(\theta_2^{'\top} x)) \right|}
{\|\theta_1 - \theta_1'\|_2 \|\theta_2 - \theta_2'\|_2} \\
&\leq \frac{4D_{\eta,m,T}\epsilon K_1}{R_0^2}.
\end{align*}
Thus, if $\| \overline{\theta}_1 - \overline{\theta}_2 \|_2 \leq \epsilon$ for $\overline{\theta}_1, \overline{\theta}_2 \in \overline{\Omega}_+$, 
then $\| g_{\overline{\theta}_1} - g_{\overline{\theta}_2}\|_\infty \leq 4D_{\eta,m,T}\epsilon K_1 / R_0^2$.
Since, 
\[ \mathcal{G}_+ \subset \{ g_{\overline{\theta}} \mid\ \| \overline{\theta}\|_2 \leq 2R + 2D_{\eta,m,T},\ \overline{\theta} \in \realsp^{2d} \}, \]
we get for the unit-ball $B_1 \subset \realsp^{2d}$ with respect to $\|\cdot \|_2$, 
\[ \mathcal{N}(\mathcal{G}_{+}|_X,\epsilon,\|\cdot\|_\infty) 
\leq C_1^d \mathcal{N}\left(B_1,\frac{R_0^2 \epsilon}{K_1(RD_{\eta,m,T} + D_{\eta,m,T}^2)},\|\cdot\|_2 \right), \]
where $C_1>0$ is a uniform constant. Hence, 
\[ \log \mathcal{N}(\mathcal{G}_{+}|_X,\epsilon,\|\cdot\|_\infty) 
\leq O\left( d\log\left( 1 + \frac{K_1(R D_{\eta,m,T} + D_{\eta,m,T}^2)}{R_0^2 \epsilon} \right) \right). \] 

Applying Lemma \ref{lemma:dudley_integral} with $\alpha = K_1/\sqrt{\ndata}$, we obtain
\begin{equation}
\radcomp( \mathcal{G}_+|_X) = O\left( K_1\sqrt{\frac{d}{\ndata} \log\left( 1 + \frac{\sqrt{\ndata}(RD_{\eta,m,T} + D_{\eta,m,T}^2)}{R_0^2} \right)}\right). \label{eq:rademacher_pos_bound}
\end{equation} 

We next evaluate $\radcomp( \mathcal{G}_{-}|_X)$ by using a linear approximation.
Since $|\sigma''(\cdot)| \leq K_2$, we get 
\[ | \sigma(\theta^{'\top} x) - \sigma(\theta^{\top} x) - \sigma'(\theta^\top x)( \theta' - \theta)^\top x  | 
\leq K_2\|\theta' - \theta \|_2^2. \]
Therefore, we get for $\overline{\theta}=(\theta,\theta') \in \overline{\Omega}_{-}$,
\[ \left| g_{\overline{\theta}}(x) - \frac{\sigma'(\theta^\top x)( \theta' - \theta)^\top x}{ \|\theta-\theta'\|_2 }\right| 
\leq K_2 \|\theta - \theta'\|_2 \leq K_2 R_0.\]

From this approximation, the Rademacher complexity can be bounded as follows.
\begin{align*}
\radcomp( \mathcal{G}_{-}|_X) 
&\leq K_2R_0 
+ \frac{1}{\ndata}\expec\left[ \sup_{\overline{\theta} \in \overline{\Omega}_{-}} \sum_{i=1}^{\ndata}\epsilon_i 
\frac{\sigma'(\theta^\top x_i)( \theta' - \theta)^\top x_i}{ \|\theta-\theta'\|_2 }\right] \\
&\leq K_2R_0 
+ \frac{1}{\ndata}\expec\left[ \sup_{\substack{\|\theta\|_2 \leq R + D_{\eta,m,T},\\ \|w\|_2 \leq 1}} \sum_{i=1}^{\ndata}\epsilon_i 
\sigma'(\theta^\top x_i)w^\top x_i\right].
\end{align*}
When $\sqrt{ \|\theta_1 - \theta_2\|_2^2 + \|w_1 - w_2\|_2^2} \leq \epsilon$ for $\|\theta_i\|_2 \leq R + D_{\eta,m,T}$ and 
$\|w_i\|_2 \leq 1$, we get for $\|x\|_2 \leq 1$,
\begin{align*}
|\sigma'(\theta_1^\top x)w_1^\top x - \sigma'(\theta_2^\top x)w_2^\top x |
&\leq | (\sigma'(\theta_1^\top x) - \sigma'(\theta_2^\top x)) w_1^\top x |
+ | \sigma'(\theta_2^\top x) ( w_1 - w_2)^\top x | \\
&\leq (K_1+K_2) \epsilon.
\end{align*}
We set 
\[ \mathcal{G}_{-}' \defeq \{ x \rightarrow \sigma'(\theta^\top x)w^\top x \mid\ \|\theta\|_2 \leq R + D_{\eta,m,T},\ \|w\|_2 \leq 1 \}. \]
Therefore, by the same argument as the case of $\mathcal{G}_{+}$, the following bound holds.
\[ \mathcal{N}(\mathcal{G}_{-}'|_X,\epsilon,\|\cdot\|_\infty) 
\leq C_2^d \mathcal{N}\left(B_1,\frac{\epsilon}{ (K_1+K_2)(R+D_{\eta,m,T})},\|\cdot\|_2 \right), \]
where $C_2>0$ is a uniform constant.
Hence, 
\[ \log \mathcal{N}(\mathcal{G}_{-}'|_X,\epsilon,\|\cdot\|_\infty) 
\leq O\left( d \log \left( 1 + \frac{ (K_1+K_2)(R+D_{\eta,m,T})}{\epsilon} \right) \right).  \]
By Lemma \ref{lemma:dudley_integral} with $\alpha = 1/\sqrt{\ndata}$, we get
\begin{equation}
\radcomp( \mathcal{G}_{-}|_X) \leq O\left( K_2 R_0 + \sqrt{\frac{d}{\ndata}\log\left( 1 + \sqrt{\ndata}(K_1+K_2)(R+D_{\eta,m,T})\right)}\right). \label{eq:rademacher_neg_bound}    
\end{equation} 

Combining (\ref{eq:rademacher_H_bound}), (\ref{eq:rademacher_pos_bound}), (\ref{eq:rademacher_neg_bound}) with $R_0 = \sqrt{d/\ndata}$, 
and Lipschiz continuity of $l_\gamma$, we obtain
\begin{equation*}
\radcomp(\mathcal{F}_1|_X) 
\leq O\left( \frac{m^{\frac{1}{2}-\beta}D_{\eta,m,T}}{\gamma} (1+K_1+K_2)\sqrt{\frac{d}{\ndata} 
\log\left( \ndata(1+K_1+K_2)(R+D_{\eta,m,T}) \right) } \right).
\end{equation*}
Now, let us turn to the second part of Proposition \ref{prop:complexity_bound}. Let us assume that the function $\sigma$ is convex and satisfies $\sigma(0)=0$. The main argument uses the convexity of activation function in the same spirit as \cite{Chinot2019}. As for the first part, with probability larger than $1-\delta$ over the initialization
\[ \max_{r\in \{1,\ldots,m\}} \|\theta_r^{(0)}\|_2 \leq R \defeq \sqrt{\frac{1}{b}\log\left( \frac{mA}{\delta}\right)}. \]
We only focus on the control of $\radcomp(\mathcal{H}|_X)$. 
\begin{align*}
	\radcomp(\mathcal{H}|_X) & = \frac{1}{ n}\expec \left[ \sup_{ \Theta \in \Omega } \sum_{i=1}^n \epsilon_i f_{\Theta}(x_i)  \right] \\
	& = \frac{1}{ n }	\expec \left[ \sup_{ \Theta \in \Omega} \sum_{i=1}^n \epsilon_i \big( f_{\Theta}(x_i) - f_{\Theta^{(0)}}(x_i)  \big) \right]\\
	& = \frac{1}{n }	\expec \left[ \sup_{ \Theta \in \Omega } \sum_{i=1}^n \epsilon_i  \sum_{r=1}^m \frac{a_r}{m^{\beta}} \big( \sigma(\theta_r^T x_i) -\sigma(\theta_r^{(0)T} x_i)   \big) \right] \\
	& = \frac{1}{n}	\expec \left[ \sup_{ \Theta \in \Omega } \sum_{(i,r) \in \mathcal A} \epsilon_i \frac{a_r}{m^{\beta}} \big( \sigma(\theta_r^T x_i) -\sigma(\theta_r^{(0)T} x_i)   \big) \right] \\
	& + \frac{1}{n}	\expec \left[ \sup_{ \Theta \in \Omega } \sum_{(i,r) \in \mathcal A^c} \epsilon_i \frac{a_r}{m^{\beta}} \big( \sigma(\theta_r^{(0)T} x_i) - \sigma(\theta_r^T x_i)  \big) \right]
\end{align*}
where $\mathcal A =  \{ (i,r) \in \{1,\cdots, n \} \times \{1,\cdots,m\}: \sigma(\theta_r^T x_i) -\sigma(\theta_r^{(0)T} x_i) \geq 0 \}$. \\

Let us control the first term (i.e for $(i,r) \in \mathcal A$). For any $i,r$ in $\mathcal A$ let $\psi_{i,r}: \mathbb R \mapsto \mathbb R$ defined for all $u \in \mathbb R$ as:
\begin{equation*}
\psi_{i,r}(u) = \sigma(u + \theta_r^{(0)^T} x_i) -  \sigma(\theta_r^{(0)^T} x_i)
\end{equation*} 
The functions $\psi_{i,r}$ are such that $\psi_{i,r}(0) = 0$. There are convex because $\sigma$ is. In particular for any $\alpha  \geq 1$ and $u \in \mathbb R$, $\psi_{i,r}(\alpha u) \geq \alpha \psi_{i,r}(u)$. We also have $\psi_{i,r}\big( (\theta_r-\theta_r^{(0)})^T x_i \big) =  \sigma(\theta_r^T x_i) -\sigma(\theta_r^{(0)T} x_i)$. Since $\Theta \in  \Omega$ we have $ \|\Theta- \Theta^{(0)} \|_{2} \leq D_{\eta,m,T}$ and for any $r \in \{1,\cdots,m  \}$, $\|\theta_r- \theta_r^{(0)} \|_2 \leq D_{\eta,m,T}$. As a consequence, for any $(i,r) \in \mathcal A$, there exists $\beta_{i,r} \in [0,1]$ such that 
\begin{equation*}
    \frac{D_{\eta,m,T}}{\|\theta_r-\theta_r^{(0)}\|_2} \psi_{i,r}\big( (\theta_r-\theta_r^{(0)})^T x_i \big) = \beta_{i,r} \psi_{i,r}\bigg( \frac{D_{\eta,m,T}}{\|\theta_r-\theta_r^{(0)}\|_2} (\theta_r-\theta_r^{(0)})^T x_i \bigg)
\end{equation*}
Since for any $i \in \{1,\cdots,n\}$, $ \sum_{r=1}^m \frac{\|\theta_r-\theta_r^{(0)}\|_2 }{C_M m^{\beta}} \beta_{i,r} \leq 1$, we get 
\begin{align*}
\frac{1}{n}	\expec  \sup_{ \Theta \in \Omega } & \sum_{(i,r) \in \mathcal A } \epsilon_i \frac{a_r}{m^{\beta}} \big( \sigma(\theta_r^T x_i) -\sigma(\theta_r^{(0)T} x_i)   \big)   \\
& \leq \frac{1}{n}\frac{C_M}{D_{\eta,m,T}} \expec \left[ \sup_{\Theta = (\theta_r)_{r=1}^m :  \|\theta_r- \theta^{(0)}_r \|_{2} \leq D_{\eta,m,T}} \sum_{(i,j)\in \mathcal A} \epsilon_i a_r \frac{\|\theta_r-\theta_r^{(0)}\|_2 }{C_M m^{\beta}} \beta_{i,r}   \psi_{i,r}\bigg( \frac{D_{\eta,m,T}(\theta_r-\theta_r^{(0)})^T x_i}{\|\theta_r-\theta_r^{(0)}\|_2}  \bigg) \right]\\
& \leq \frac{1}{n}\frac{C_M}{D_{\eta,m,T}} 	\expec \left[ \sup_{\begin{subarray}{l}\Theta = (\theta_r)_{r=1}^m :  \|\theta_r \|_{2} \leq D_{\eta,m,T}\\
	b = (b_r)_{r=1}^m, b_r \in [0,1], \sum_{r=1}^m b_r \leq 1 \end{subarray}} \sum_{i=1}^n \epsilon_i \sum_{r=1}^m a_r b_r  \psi_{i,r}\big( \theta_r^T x_i \big) \right]\\
 & =  \frac{1}{n}\frac{C_M}{D_{\eta,m,T}} 	\expec \left[ \sup_{\begin{subarray}{l}\Theta = (\theta_r)_{r=1}^m :  \|\theta_r \|_{2} \leq D_{\eta,m,T}\\
	b = (b_r)_{r=1}^m, b_r \in [0,1], \sum_{r=1}^m b_r \leq 1 \end{subarray}} \sum_{i=1}^n \epsilon_i \sum_{r=1}^m a_r b_r \big( \sigma((\theta_r+\theta_r^{(0)})^T x_i)  - \sigma(\theta_r^{(0)T}x_i) \big) \right] \enspace. \\
\end{align*}	
Therefore, with probability larger than $1-\delta$,
\begin{align*}	
\frac{1}{n}	\expec  \sup_{ \Theta \in \Omega } & \sum_{(i,r) \in \mathcal A } \epsilon_i \frac{a_r}{m^{\beta}} \big( \sigma(\theta_r^T x_i) -\sigma(\theta_r^{(0)T} x_i)   \big)   \\ & \leq  \frac{1}{n} \frac{C_M}{D_{\eta,m,T}} 	\expec \left[ \sup_{\begin{subarray}{l}\|\theta_r \|_{2} \leq D_{\eta,m,T}; \|\tilde \theta_r\|_2 \leq R  \\
	b = (b_r)_{r=1}^m, b_r \in [0,1], \sum_{r=1}^m b_r \leq 1 \end{subarray}} \sum_{i=1}^n \epsilon_i \sum_{r=1}^m a_r b_r \big( \sigma((\theta_r+\tilde \theta_r)^T x_i)  - \sigma( \tilde \theta_r^T x_i) \big) \right]\\
&  \leq \frac{1}{n} \frac{ C_M}{D_{\eta,m,T}} 	\expec  \left[ \sup_{\theta:  \|\theta \|_{2} \leq D_{\eta,m,T}; \|\tilde \theta_r\|_2 \leq R  } \sum_{i=1}^n \epsilon_i \big(\sigma((\theta+ \tilde \theta)^T x_i) + \sup_{\|\tilde \theta_r\| \leq R  } \sum_{i=1}^n \epsilon_i \sigma(\tilde{\theta}^Tx_i) \big) \right] \\
& \leq \frac{K_1}{n} \frac{ C_M}{D_{\eta,m,T}} 	\expec  \left[ \sup_{\theta:  \|\theta \|_{2} \leq D_{\eta,m,T}; \|\tilde \theta_r\|_2 \leq R  } \sum_{i=1}^n \epsilon_i (\theta + \tilde \theta)^T x_i + \sup_{\|\tilde \theta_r\|_2 \leq R  } \sum_{i=1}^n \epsilon_i \tilde{\theta}^T x_i  \right] \\
& \leq\frac{K_1}{\sqrt{n}} \frac{ C_M}{D_{\eta,m,T}}(D_{\eta,m,T} + 2R) = \frac{K_1  m^{1/2-\beta}(D_{\eta,m,T} + 2R)}{\sqrt n} \enspace.
\end{align*}
Let us turn to the second term. For any $(i,r)$ in $\mathcal A^c$ let $\tilde \psi_{i,r}: \mathbb R \mapsto \mathbb R$ defined for all $u \in \mathbb R$ as:
\begin{equation*}
\tilde \psi_{i,r}(u) = \sigma(u + \theta_r^T x_i) -  \sigma(\theta_r^T x_i)
\end{equation*} 
We have $\tilde \psi_{i,r}\big( (\theta_r^{(0)}-\theta_r)^T x_i \big) =   \sigma(\theta_r^{(0)T} x_i)-\sigma(\theta_r^T x_i) $. Using the same path as for $(i,j) \in \mathcal A$, with probability larger than $1-\delta$, we obtain,
\begin{align*}
\frac{1}{n}	\expec  \sup_{ \Theta \in \Omega } & \sum_{(i,r) \in \mathcal A^c } \epsilon_i \frac{a_r}{m^{\beta}} \big(  \sigma(\theta_r^{(0)T} x_i) - \sigma(\theta_r^T x_i)  \big)   \\
 & \leq  \frac{1}{n}\frac{C_M}{D_{\eta,m,T}} 	\expec \left[ \sup_{\begin{subarray}{l}\|\theta_r \|_{2} \leq D_{\eta,m,T}, \|\tilde \theta_r \|_2 \leq R+D_{\eta,m,T} \\
	b = (b_r)_{r=1}^m, b_r \in [0,1], \sum_{r=1}^m b_r \leq 1 \end{subarray}} \sum_{i=1}^n \epsilon_i \sum_{r=1}^m a_r b_r \big( \sigma((\theta_r+ \tilde \theta_r)^T x_i)  - \sigma(\tilde \theta_r^T x_i) \big) \right] \enspace. \\
& \leq\frac{K_1}{\sqrt{n}} \frac{ C_M}{D_{\eta,m,T}}(3D_{\eta,m,T} + 2R) \enspace.
\end{align*}
\end{proof}